\documentclass[twoside]{article}

\usepackage[accepted]{aistats2021}

\usepackage[round]{natbib}

\usepackage{amsmath,amssymb,amsthm}
\usepackage{mathrsfs}
\usepackage{enumitem}

\usepackage{graphicx}
\usepackage{booktabs}
\usepackage{xcolor}
\usepackage{bbm}
\usepackage{verbatim}
\usepackage[colorlinks,citecolor=cyan,linkcolor=blue]{hyperref}
\usepackage[nameinlink,capitalize]{cleveref}
\crefformat{equation}{#2 Equation~(#1) #3}
\usepackage{thm-restate}

\allowdisplaybreaks
\usepackage[framemethod=TikZ]{mdframed}

\newtheoremstyle{definition}
{3pt} 
{3pt} 
{} 
{} 
{\bfseries} 
{.} 
{.5em} 
{} 

\usepackage{selectp}

\theoremstyle{definition}

\newtheorem{theorem}{Theorem}[section]

\newtheorem{corollary}[theorem]{Corollary}
\newtheorem{definition}[theorem]{Definition}

\newtheorem{lemma}[theorem]{Lemma}

\newtheorem{assumption}[theorem]{Assumption}

\newcommand{\statespace}{\mathcal{X}}
\newcommand{\actionspace}{\mathcal{A}}
\newcommand{\repdim}{K}
\newcommand{\repix}{k}

\begin{document}

%

%

\twocolumn[

\aistatstitle{On The Effect of Auxiliary Tasks on Representation Dynamics}

\aistatsauthor{ Clare Lyle* \And Mark Rowland* \And Georg Ostrovski \And Will Dabney}

\aistatsaddress{ University of Oxford \And DeepMind \And DeepMind \And DeepMind } ]

\begin{abstract}
While auxiliary tasks play a key role in shaping the representations learnt by reinforcement learning agents, much is still unknown about the mechanisms through which this is achieved.
This work develops our understanding of the relationship between auxiliary tasks, environment structure, and representations by analysing the dynamics of temporal difference algorithms.
Through this approach, we establish a connection between the spectral decomposition of the 
transition operator and the representations induced by a variety of auxiliary tasks.
We then leverage insights from these theoretical results to inform the selection of auxiliary tasks for deep reinforcement learning agents in sparse-reward environments.
\end{abstract}

\section{Introduction}

Auxiliary tasks have provided robust benefits to deep reinforcement learning agents \citep{jaderberg2016reinforcement,mirowski2017learning, lin2019adaptive}. A commonly-held belief is that these benefits are mediated through improved representation learning.
This hypothesis naturally raises a number of questions that, broadly speaking, remain open. What makes a good auxiliary task? Can we predict how an auxiliary task will affect an agent's representation? When should one auxiliary task be used instead of another? More generally, how should this hypothesis about the mechanism of auxiliary tasks itself be tested? The complex interacting components of large-scale deep reinforcement learning agents make it difficult to extract general insights. 
In this work we aim to shed light on the answers to these questions by distilling the benefits of auxiliary tasks down to the effects on the dynamics of the representations of reinforcement learning agents.

\begin{figure}[t!]
    \centering
    \null\hfill
    \includegraphics[keepaspectratio,width=.46\textwidth]{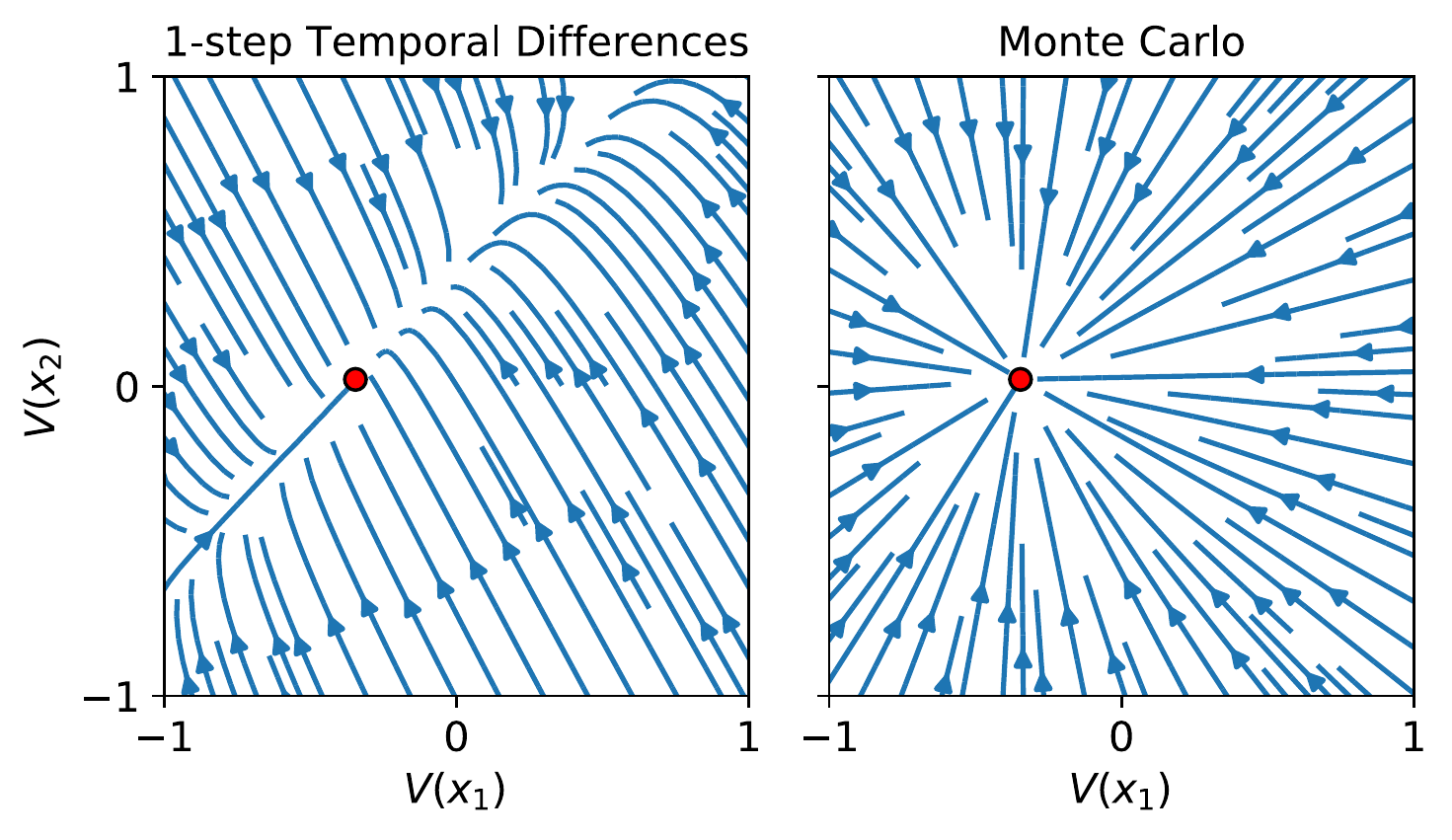}\hfill
    \null
    \caption{An example of qualitatively different value function dynamics for a two-state MDP for 1-step temporal difference learning and Monte Carlo learning, with fixed point $V^\pi$ in red.}
    \label{fig:two-state-example}
\end{figure}

We begin by considering a \emph{learning dynamics} framework for studying the effects of auxiliary tasks; see Figure~\ref{fig:two-state-example} for a toy illustration, with full details given in Section~\ref{sec:learning-dynamics}. The central idea behind this framework is that it is not just \emph{what} an agent learns that dictates how its representation is shaped, but \emph{how} it learns. 

This framework provides a \emph{model} for representation learning in RL. Under this model, even in the case of value-based algorithms, it is shown that agents automatically incorporate the transition structure of the environment into their representations. We characterize the dynamics induced by a number of auxiliary tasks, with particular focus on ensemble predictions and random cumulant functions, and prove convergence of the induced representations to subspaces defined by certain decompositions of the environment's transition operator.
We then consider the effectiveness of auxiliary tasks in sparse-reward environments, and via the use of the learning dynamics framework, construct a hypothesis as to which auxiliary tasks should be particularly well suited to such environments; we then test these developments in the Arcade Learning Environment \citep{bellemare2013arcade}, demonstrating strong performance with random cumulant auxiliary tasks.

\section{Background}

We consider a Markov decision process $(\mathcal{X}, \mathcal{A}, P, R)$ comprising a finite state space $\mathcal{X}$, finite action space $\mathcal{A}$, transition kernel $P : \mathcal{X} \times \mathcal{A} \rightarrow \mathscr{P}(\mathcal{X})$, and reward distribution function $\mathcal{R} : \mathcal{X} \times \mathcal{A} \rightarrow \mathscr{P}(\mathbb{R})$.

\subsection{Value-based reinforcement learning}

Two key tasks in reinforcement learning are \emph{policy evaluation} and \emph{policy optimisation}. The former is specified by a policy $\pi: \mathcal{X} \rightarrow \mathscr{P}(\mathcal{A})$. An agent using the policy $\pi$ to interact with the environment generates a sequence of states, actions and rewards $(X_t, A_t, R_t)_{t \geq 0}$. The performance of the agent is summarised by the return $\sum_{t \geq 0} \gamma^t R_t$, for a discount factor $\gamma \in [0,1)$. The goal of policy evaluation is to (approximately) compute the value function
\begin{align*}
    \textstyle
    V^\pi(x) = \mathbb{E}_\pi[ \sum_{t \geq 0} \gamma^t R_t \mid X_0 = x ] \, ,
\end{align*}
for all $x \in \mathcal{X}$. Policy optimisation consists of finding a policy $\pi^* \in \mathscr{P}(\mathcal{A})^\mathcal{X}$ that maximises the expected return from all possible initial states. The value function associated with $\pi^*$ is denoted $V^*$.

Crucial to the value-based approach to reinforcement learning are the \emph{Bellman operators}. The one-step evaluation operator associated with a policy $\pi$ is the function $T^\pi : \mathbb{R}^{\statespace} \rightarrow \mathbb{R}^{\statespace}$ defined by
\begin{align*}
    (T^\pi V)(x, a) = \mathbb{E}_\pi[ R_0 + \gamma V(X_1) \mid X_0 =x] \, .
\end{align*}
Introducing the transition operator $P^\pi \in \mathbb{R}^{\statespace\times\statespace}$ defined by $P^\pi(x'|x) = \sum_{a \in \mathcal{A}}\pi(a|x)P(x'|x, a)$, and the expected reward vector $R^\pi \in \mathbb{R}^{\statespace}$ defined by $R^\pi(x) = \mathbb{E}_\pi[R_0|X_0=x]$, this can be expressed even more succinctly in operator notation as
\begin{align*}
    T^\pi V = R^\pi + \gamma P^\pi V \, .
\end{align*}
The Bellman optimality operator is the function $T^* : \mathbb{R}^{\statespace\times\actionspace} \rightarrow \mathbb{R}^{\statespace\times\actionspace}$ defined by
\begin{align*}
    (T^* V)(x) \!=\! \max_{a \in \mathcal{A}}\mathbb{E}[R_0 \!+\! \gamma V(X_1) | X_0 = x, A_0 = a] \, .
\end{align*}

Repeated application of $T^\pi$ (resp., $T^*$) to any initial value function converges to $V^\pi$ (resp., $V^*$) \citep{bertsekas1996neuro}. Popular algorithms such as Q-learning \citep{watkins1992q}, which form the basis of many deep RL agents \citep{mnih2015human}, can be viewed as approximating the iterative application of $T^*$ and related operators \citep{tsitsiklis1994asynchronous,jaakola1994convergence,bertsekas1996neuro}.

\subsection{Features and representations}\label{sec:reps}

In many environments, it is impractical to store a value function as a table indexed by states, and further this does not permit generalisation in the course of learning. Instead, it is typical to parametrise $V \in \mathbb{R}^{\statespace}$ through a \emph{feature map} $\phi : \mathcal{X} \rightarrow \mathbb{R}^\repdim$ and \emph{weight vector} $\mathbf{w} \in \mathbb{R}^\repdim$, leading to a factorisation of the form
\begin{align*}
    V(x) = \langle \phi(x), \mathbf{w} \rangle \, .
\end{align*}
Such a parametrisation may be amenable to more efficient learning, for example if $\phi$ abstracts away unimportant information, allowing for generalisation between similar states.
Even more concisely, writing $\Phi \in \mathbb{R}^{\statespace\times\repdim}$ for the matrix with rows $\phi(x)$ yields
\begin{align}\label{eq:q-phi-w}
    V = \Phi \mathbf{w} \, .
\end{align}
The quantity $\Phi$ is often referred to as the agent's \emph{representation} of the environment \citep{boyan1999least, levine2017shallow, bertsekas2018feature,chung2018two,bellemare2019geometric,dabney2020value}.
In many small- and medium-scale applications, the representation is fixed ahead of time,
and only $\mathbf{w}$ is updated during learning; this is the linear function approximation regime. Many common choices of features relate to various decompositions of operators associated with the transition operators $P^\pi$. In deep reinforcement learning, however, $\Phi$ and $\mathbf{w}$ are learnt simultaneously.

\subsection{Representation learning and auxiliary tasks}

To perfectly express a value function $V^\pi$ in the form of Equation~\eqref{eq:q-phi-w}, the following condition is necessary:
\begin{align}\label{eq:value-constraint}
    \langle \Phi \rangle \supseteq \langle V^\pi \rangle \, ,
\end{align}
where $\langle \Phi \rangle$ denotes the column span of $\Phi$, and $\langle V^\pi\rangle$ denotes the one-dimensional subspace of $\mathbb{R}^{\mathcal{X}}$ spanned by $V^\pi$. However, this condition is not sufficient for efficient sample-based learning \citep{du2019good}. There are several reasons for this; for some intuition, consider that since value functions are typically learnt through bootstrapping algorithms, the agent is required to accurately express a \emph{sequence} of value functions as its estimates are updated, and thus a good representation should also allow such intermediate value functions to be expressed in the course of learning \citep{dabney2020value}.

Despite the importance of the representation $\Phi$, it remains unclear how exactly the notion of a good representation in this sense should be formalised. In spite of this, representation learning is a hugely important aspect of deep reinforcement learning.
A consistent finding in empirical deep RL research is that requiring the agent to use its representation to predict other functions of state, referred to as \emph{auxiliary tasks}, in addition to its primary task of learning an optimal policy, can lead to considerable boosts in performance. Examples of commonly-used auxiliary tasks include predicting the expected return associated with other reward functions \citep{sutton2011horde}, other discount factors \citep{fedus2019hyperbolic}, and other policies \citep{dabney2020value}, as well as other properties of the return distribution \citep{bellemare2017distributional}  and other aspects of the environment observations \citep{jaderberg2016reinforcement}, amongst others.
We discuss prior work on auxiliary tasks in greater detail in Section~\ref{sec:related-work}.
A popular hypothesis is that auxiliary tasks add further constraints to Expression~\eqref{eq:value-constraint}, requiring the representation $\Phi$ to contain more functions of interest than just $V^\pi$ in its column span \citep{bellemare2019geometric,dabney2020value}.

\section{Learning dynamics}\label{sec:learning-dynamics}

Our aim in the remainder of the paper is to develop an understanding of the ways in which auxiliary tasks shape representations in RL. Our central results establish connections between decompositions of transition operators, commonly used in static feature selection, and certain classes of auxiliary tasks used in deep reinforcement learning. To build up to these results, in this section we examine learning algorithms in the absence of auxiliary tasks, first considering tabular learning algorithms, and then moving to the case where representations and feature weights are learnt simultaneously.

\subsection{Warm-up: Tabular value function dynamics}\label{sec:value-function}

We consider the following one-step temporal difference (TD) continuous-time learning dynamics:
\begin{align*}
    \partial_t V_t(x) = \mathbb{E}_\pi[R_0 + \gamma V_t(X_1)|X_0 = x]- V_t(x) \, ,
\end{align*}
for each $x \in \mathcal{X}$, which may also be written 
\begin{align*}
    \partial_t V_t(x) = R^\pi(x) + \gamma (P^\pi V_t)(x) - V_t(x) \, ,
\end{align*}
or in full matrix notation,
\begin{align}\label{eq:value-function-ode}
    \partial_t V_t = -(I - \gamma P^\pi) V_t + R^\pi \, .
\end{align}
The differential equation in \eqref{eq:value-function-ode} is an affine autonomous system, and is straightforwardly solvable.
\begin{restatable}{lemma}{lemODESoln}\label{lem:ode-soln}
If $(V_t)_{t \geq 0}$ satisfies Equation~\eqref{eq:value-function-ode} with initial condition $V_0$ at time $t=0$, then we have
\begin{align}\label{eq:value-function-ode-solution}
    V_t = \exp( -t (I - \gamma P^\pi ) )(V_0 - V^\pi) + V^\pi \, .
\end{align}
\end{restatable}
We recover as a straightforward corollary the well-known result that $V_t \rightarrow V^\pi$ as $t \rightarrow \infty$, since all eigenvalues of $(I - \gamma P^\pi)$ have strictly positive real part. 

However, the solution in Equation~\eqref{eq:value-function-ode-solution} also describes the \emph{trajectory} by which $V_t$ reaches this limiting value.  Figure~\ref{fig:two-state-example} provides an illustration of this in a small MDP; the value functions accumulate along a particular affine subspace of $\mathbb{R}^\mathcal{X}$ prior to convergence.

This phenomenon can in fact be formalised. To do so, we need a notion of distance between subspaces of $\mathbb{R}^{\mathcal{X}}$. 
The following definition follows \citet{ye2016schubert}. Intuitively, it can be thought of as generalizing the notion of an angle between vectors to subspaces.

\begin{definition}\label{def:grassmann-distance}
For two $K$-dimensional subspaces $Y_1, Y_2 \leq \mathbb{R}^\mathcal{X}$, the \emph{principal angles} $\theta_1, \ldots, \theta_K \in [0, \pi/2]$ between the subspaces are defined by taking orthonormal matrices $\mathbf{Y}_1 \in \mathbb{R}^{\mathcal{X} \times K}$ and $\mathbf{Y}_2 \in \mathbb{R}^{\mathcal{X} \times K}$ the columns of which span $Y_1$ and $Y_2$ respectively, and defining $\theta_k = \cos^{-1}(\sigma_k(\mathbf{Y}_1^\top \mathbf{Y}_2))$, where $\sigma_k(\mathbf{A})$ is the $k$\textsuperscript{th} singular value of the matrix $\mathbf{A}$. One can check that this definition is independent of the matrices $\mathbf{Y}_1$ and $\mathbf{Y}_2$, depending only on the subspaces $Y_1,Y_2$ themselves. The \emph{Grassmann distance} $d(Y_1, Y_2)$ between $Y_1$ and $Y_2$ is then defined as $\| \theta \|_2 = (\sum_{k=1}^K \theta_k^2)^{1/2}$.
\end{definition}

With these definitions in hand, we now give a precise version of the statement alluded to in the discussion and figure above. We make some simplifying assumptions to avoid focusing on technicalities here, and give a discussion of the more general case in  Appendix~\ref{sec:more-general-value-function-results}.

\begin{assumption}\label{assume:value-function-conditions}
    $P^\pi$ is real-diagonalisable, with strictly decreasing eigenvalue sequence $1=|\lambda_1| > |\lambda_2| > \cdots >  |\lambda_{|\mathcal{X}|}|$, and corresponding right-eigenvectors $U_1, \ldots, U_{|\mathcal{X}|}$. 
\end{assumption}

\begin{restatable}{proposition}{propOneValueFunction}\label{prop:one-value-function}
    Under Assumption~\ref{assume:value-function-conditions}, and $(V_t)_{t \geq 0}$ the solution to Equation~\eqref{eq:value-function-ode}, for almost every\footnote{In the measure-theoretic sense that the set of excluded initial  conditions $V_0$ has Lebesgue measure $0$.} initial condition $V_0$, we have
    \begin{align*}
        d(\langle V_t - V^\pi \rangle, \langle U_1 \rangle) \rightarrow 0 \, .
    \end{align*}
\end{restatable}

The behaviour described by Proposition~\ref{prop:one-value-function} is exhibited in Figure~\ref{fig:two-state-example}, as the value function $V_t$ approaches the affine subspace in direction $(1, 1)$ prior to converging to $V^\pi$.
A more general version of this statement can also be given with an ensemble of $\repdim$ value functions, which indicates that yet more information about the environment is contained in the learnt collection. The proofs of these results relate to the classical \emph{power method} in linear algebra.

\begin{restatable}{proposition}{propManyValueFunctions}\label{prop:many-value-functions}
    Under Assumption~\ref{assume:value-function-conditions}, and $(V^{(\repix)}_t)_{t \geq 0}$ the solution to Equation~\eqref{eq:value-function-ode} for each $\repix=1,\ldots,\repdim$, for almost every initial condition $(V_0^{(\repix)})_{\repix=1}^\repdim$, we have
    \begin{align*}
        d(\langle V^{(\repix)}_t - V^\pi \mid \repix \in [\repdim] \rangle, \langle U_{1:\repdim} \rangle) \rightarrow 0 \, .
    \end{align*}
\end{restatable}

\vspace{2mm}
\mdfsetup{%
backgroundcolor=black!10,
roundcorner=10pt}
\begin{mdframed}
\textbf{Key insight.}\\
Even in an environment with no reward signal at all (in which case $V^\pi = 0$), an agent performing TD learning still picks up information about the transition structure of the environment within its value function.
\end{mdframed}

Due to the importance of the vectors $U_{1:K}$ in this analysis, we introduce the term \emph{eigen-basis functions} (EBFs) to describe them.

We observe that a similar analysis, indicating similar behaviour, is possible for related learning algorithms such as $n$-step temporal difference learning and TD($\lambda$); see Appendix~\ref{sec:beyond-one-step} for further details. In contrast, Monte Carlo learning dynamics correspond to the differential equation
\begin{align*}
    \partial_t V_t = (I - \gamma P^\pi)^{-1}R^\pi - V_t \, ,
\end{align*}
which has the solution
\begin{align*}
    V_t = e^{-t}(V_0 - (I - \gamma P^\pi)^{-1} R^\pi) + (I - \gamma P^\pi)^{-1} R^\pi \, .
\end{align*}
The trajectory associated with this solution simply linearly interpolates between $V_0$ and $V^\pi$, as illustrated in Figure~\ref{fig:two-state-example}, and does not pick up any additional information about the environment in the value function as learning proceeds.  See Appendix~\ref{sec:beyond-one-step} for further details.
This example serves to illustrate that it is not just \emph{what} an agent learns ($V^\pi$), but \emph{how} the agent learns that plays a key, measurable role in what environment information is picked up in its value function. We now apply this perspective to representation learning. 

\subsection{Representation dynamics}\label{sec:rep-dynamics}

Recall the parametrisation of $V \in \mathbb{R}^{\mathcal{X}}$ from Section~\ref{sec:reps}, taking the form
\begin{align*}
    V = \Phi \mathbf{w} \, ,
\end{align*}
for $\Phi \in \mathbb{R}^{\mathcal{X} \times \repdim}$, $\mathbf{w} \in \mathbb{R}^{\repdim}$. 
Central to deep reinforcement learning agents is the idea that $\Phi$ and $\mathbf{w}$ are simultaneously learnt from a single RL loss. 
As in the value function case, we will focus on the dynamics with single-step temporal difference learning; remarks on other learning algorithms are given in Appendix~\ref{sec:beyond-one-step}. 
The dynamics associated with single-step TD learning are given by
\begin{align}
    \partial_t \Phi_t & = -\alpha \frac{1}{2}\nabla_{\Phi_t} \| R^\pi + \texttt{SG}[\gamma P^\pi \Phi_t \mathbf{w}_t] - \Phi_t \mathbf{w}_t \|^2_2 \, , \label{eq:phi-ode} \\
    \partial_t \mathbf{w}_t & = -\beta\frac{1}{2}\nabla_{\mathbf{w}_t} \| R^\pi + \texttt{SG}[\gamma P^\pi \Phi_t \mathbf{w}_t] - \Phi_t \mathbf{w}_t \|^2_2 \label{eq:w-ode} \, ,
\end{align}
where $\alpha, \beta \in [0, \infty)$ are learning rates, implying that features and weights may be learnt at different rates.
Further, $\texttt{SG}[]$ denotes a \emph{stop-gradient}, indicating that we treat the instances of $\Phi_t$ and $\mathbf{w}_t$ within the expression as constants with regard to computing derivatives; this reflects the fact that temporal difference learning is a \emph{semi-gradient} method.

The use of a single loss to learn both the representation and weights corresponds to the approach taken in deep RL, and we will use these dynamics as an idealized model of the deep RL setting. While this model ignores some practicalities of deep RL (such as visitation distributions, implicit bias from the function approximation architecture, and stochasticity introduced by mini-batch training), it allows us to obtain valuable insights into representation dynamics which, as we shall see in Section \ref{sec:experiments}, accurately predict the behaviour of deep RL agents. 

\begin{restatable}{lemma}{lemCoupledDynamics}
Let $\Phi_t$ and $\mathbf{w}_t$ parameterize a value function approximator as defined above. Then
\begin{align}
    \partial_t \Phi_t & = \alpha (R^\pi + \gamma P^\pi \Phi_t \mathbf{w}_t - \Phi_t \mathbf{w}_t) \mathbf{w}_t^\top \, , \label{eq:phi-flow} \\
    \partial_t \mathbf{w}_t & = \beta \Phi_t^\top(R^\pi + \gamma P^\pi \Phi_t \mathbf{w}_t - \Phi_t \mathbf{w}_t) \label{eq:w-flow} \, .
\end{align}
\end{restatable}

This joint flow on $\Phi_t$ and $\mathbf{w}_t$ leads to much richer behaviour than the flow considered on value functions in the previous section. Without further assumptions, the evolution of the representation $\Phi_t$ may be complex, and will not necessarily incorporate environment information as described for the case of value functions in Proposition~\ref{prop:many-value-functions}.
In particular, in sparse reward environments, the agent may learn to predict a near-zero value function by setting the weights $\mathbf{w}_t$ close to zero, which would effectively prevent any further updating of the features $\Phi_t$, ruling out the possibility of a result analogous to Proposition~\ref{prop:many-value-functions}.

\section{Auxiliary task dynamics}\label{sec:aux-dynamics}

Having studied the temporal difference learning dynamics in Equation~\eqref{eq:phi-flow} \& \eqref{eq:w-flow}, we now examine how auxiliary value-prediction tasks influence the behaviour of the agent's representation during the learning process.

As described above, developing a granular description of the joint learning dynamics of the representation and weights of the learner is a complex task, and so we focus on the limiting case in which the number of auxiliary tasks is large relative to the dimensionality of the representation. We conclude that under certain conditions, representations learnt in the many-task limit bear a close connection with the \emph{eigen-basis functions} described in Section~\ref{sec:value-function}, and also \emph{resolvent singular basis functions}, a new decomposition introduced in Section~\ref{sec:random-cumulants}. The reader may find it useful to refer to Appendix~\ref{sec:feature-selection} for a more detailed discussion of these decompositions.

\begin{table*}
    \centering
    \begin{tabular}{c|c|c|c|c}
        \toprule
        Auxiliary task & Dynamics ($r=0$) & $\Phi_\infty$ ($r=0$)  & $\Phi_\infty$ ($r\neq 0$) & Limit of $\langle \Phi_t - \Phi_\infty \rangle$\\
        \midrule
        Ensemble & $-(I - \gamma P^\pi) \Phi_t$ & $0$ & $(I - \gamma P^\pi)^{-1} r\epsilon^\top$ & EBFs of $P^\pi$\\
        Random cumulants & $-(I - \gamma P^\pi) \Phi_t + Z_\Sigma$ & $(I - \gamma P^\pi)^{-1} Z_\Sigma$ &   $(I - \gamma P^\pi)^{-1} Z_\Sigma$ & EBFs of $P^\pi$\\
        Additional policies & $  -(I - \gamma P^{\overline{\pi}}) \Phi_t$ & $0$ & $ (I - \gamma P^{\overline{\pi}})^{-1} R^{\overline{\pi}} \epsilon^\top$ & EBFs of $P^{\bar{\pi}}$\\
        Multiple $\gamma$s & $-(I - \overline{\gamma} P^\pi)\Phi_t$ & $0$  & $(I - \overline{\gamma} P^{\pi})^{-1}R^{\pi} \epsilon^\top$ & EBFs of $P^{\pi}$\\
        \bottomrule
    \end{tabular}
    \caption{Summary of dynamics and limiting solutions under some common auxiliary tasks in the limit of infinitely-many prediction outputs. For additional policies, $\overline{\pi}$ denotes the average of the finite set of policies $\pi_1,\ldots,\pi_L$ under consideration ($L$ fixed and independent of $M$), and for multiple discount factors, $\overline{\gamma}$ denotes the average of the discount factors $\gamma_1,\ldots,\gamma_L$ under consideration.
    }
    \label{tab:theory}
\end{table*}
\subsection{Ensemble value prediction}

We begin by considering the auxiliary task of \emph{ensemble value prediction} \citep{osband2016deep, anschel2017averaged, agarwal2019striving}. Rather than making a single prediction of the value function $Q^\pi$, the learner makes $M \in \mathbb{N}$ separate predictions as linear functions of a common representation $\Phi^{M}$, using $M$ independently initialized weights matrices $\mathbf{w}^{m} \in \mathbb{R}^{\repdim}$ ($m=1,\ldots,M$). We note that while at initialization $\Phi^M_0 \in \mathbb{R}^{\statespace \times d}$ is independent of $M$, its dynamics do depend on $M$ through the contribution of the weights. Simultaneous temporal difference learning on all predictions leads to the following dynamics:
\begin{align}
    \partial_t \Phi^{M}_t  \label{eq:ensemble-phi-flow}
    \!=  &   \alpha\! \sum_{m=1}^M (R^\pi\! +\! \gamma P^\pi \Phi^{M}_t \mathbf{w}_t^{m}\! -\! \Phi^{M}_t \mathbf{w}_t^{m})  (\mathbf{w}_t^{m} )^\top \, , \\
    \partial_t \mathbf{w}_t^{m} = & \beta (\Phi^{M}_t)^\top (R^\pi + \gamma P^\pi \Phi^M_t \mathbf{w}^{m}_t - \Phi_t \mathbf{w}^{m}_t ) \, . \label{eq:ensemble-w-flow}
\end{align}

The following result characterises the representation learnt by the agent in the many-tasks limit, again establishing a connection to EBFs; we follow the approach described by \citet{arora2019fine} in fixing the linear weights associated with the value function; this dramatically simplifies our analysis, while still describing practical settings in which the features and weights are trained separately as in \citet{chung2018two}.

\begin{restatable}{theorem}{thmInfiniteHeads}\label{thm:infinite-heads}

For $M \in \mathbb{N}$, let $(\Phi^{M}_t)_{t \geq 0}$ be the solution to Equation~\eqref{eq:ensemble-phi-flow}, with each $\mathbf{w}^{m}_t$ for $m=1,\ldots,M$ initialised independently from $N(0, \sigma_M^2)$, and fixed throughout training ($\beta=0$). 
We consider two settings: first, where the learning rate $\alpha$ is scaled as $\frac{1}{M}$
and $\sigma_M^2 = 1$ for all $M$, and second where $\sigma_M^2 = \frac{1}{M}$ and the learning rate $\alpha$ is equal to $1$. These two settings yield the following dynamics, respectively:
\begin{align}
       \lim_{M \rightarrow \infty} \partial_t \Phi_t^{M} \overset{P}{=}& -(I - \gamma P^\pi )\Phi_t^{M}\quad  \text{, and } \\
       \lim_{M \rightarrow \infty} \partial_t \Phi_t^{M} \overset{D}{=}& -(I - \gamma P^\pi )\Phi_t^{M} + R^\pi \epsilon^\top \; \text{,  $\epsilon \sim \mathcal{N}(0, I)\,$.}
\end{align}
The corresponding limiting trajectories for a fixed initialisation $\Phi_0 \in \mathbb{R}^{\statespace\times \repdim}$, are therefore given respectively by
\begin{align}
        \lim_{M \rightarrow \infty} \Phi_t^{M} \overset{P}{=}& \exp(-t(I - \gamma P^\pi))\Phi_0 \quad  \text{, and } \\
         \lim_{M \rightarrow \infty} \Phi_t^{M}  \overset{D}{=}& \exp(-t(I - \gamma P^\pi))(\Phi_0 - (I - \gamma P^\pi)^{-1} R^\pi \varepsilon^\top ) \nonumber \\
        & \qquad \ + (I - \gamma P^\pi)^{-1}R^\pi \varepsilon^\top  \, ,\,  \epsilon \sim \mathcal{N}(0, I)\,.
\end{align}
\end{restatable}

In contrast to the case described in Section~\ref{sec:rep-dynamics}, this result indicates that the introduction of auxiliary tasks leads to useful environment information being incorporated into the representation.
Indeed, the dynamics described above imply the following convergence result, analogous to Proposition~\ref{prop:many-value-functions}.

\begin{restatable}{corollary}{propSubspaceConvergence}\label{prop:subspaceconvergence}
Under the feature flow \eqref{eq:ensemble-phi-flow} with $\mathbf{w}^m_t$ fixed at initialization for each $i = 1, \dots, M$ and Assumption~\ref{assume:value-function-conditions}, for almost all initialisations $\Phi_0$, we have when $R^\pi = 0$
    \begin{align*}
        d(\langle \Phi_t \rangle, \langle U_{1:\repdim} \rangle) \rightarrow 0 \, ,\quad\text{as } t \rightarrow \infty.
    \end{align*}
\end{restatable}

\vspace{2mm}
\mdfsetup{%
backgroundcolor=black!10,
roundcorner=10pt}
\begin{mdframed}
\textbf{Key insight.}\\
Under the conditions of Theorem~\ref{thm:infinite-heads} and Corollary~\ref{prop:subspaceconvergence}, the ensemble auxiliary tasks cause the agent's representation $\Phi$ to align with EBFs.
\end{mdframed}

We show in Appendix \ref{sec:ensemble-dynamics} that this behaviour is observed in practice when $M \gg \repdim$ and the value of $\mathbf{w}^m_t$ is fixed at initialization. We additionally compare the representations learned when $\mathbf{w}^m_t$ is allowed to vary over training. Here we find empirically that allowing the weights to vary during training induces dynamics that differ from those predicted by Theorem~\ref{thm:infinite-heads} for the fixed-weights setting.
To illustrate this, we follow the evolution of a single column of $\Phi_t$, i.e. a single feature vector $\phi_t$, trained with the ensemble prediction dynamics of Equations~\eqref{eq:ensemble-phi-flow} \& \eqref{eq:ensemble-w-flow} on a simple four-rooms gridworld environment in
Figure~\ref{fig:feature-viz}.

\begin{figure}[!b]
    \centering
    \includegraphics[width=0.38\textwidth]{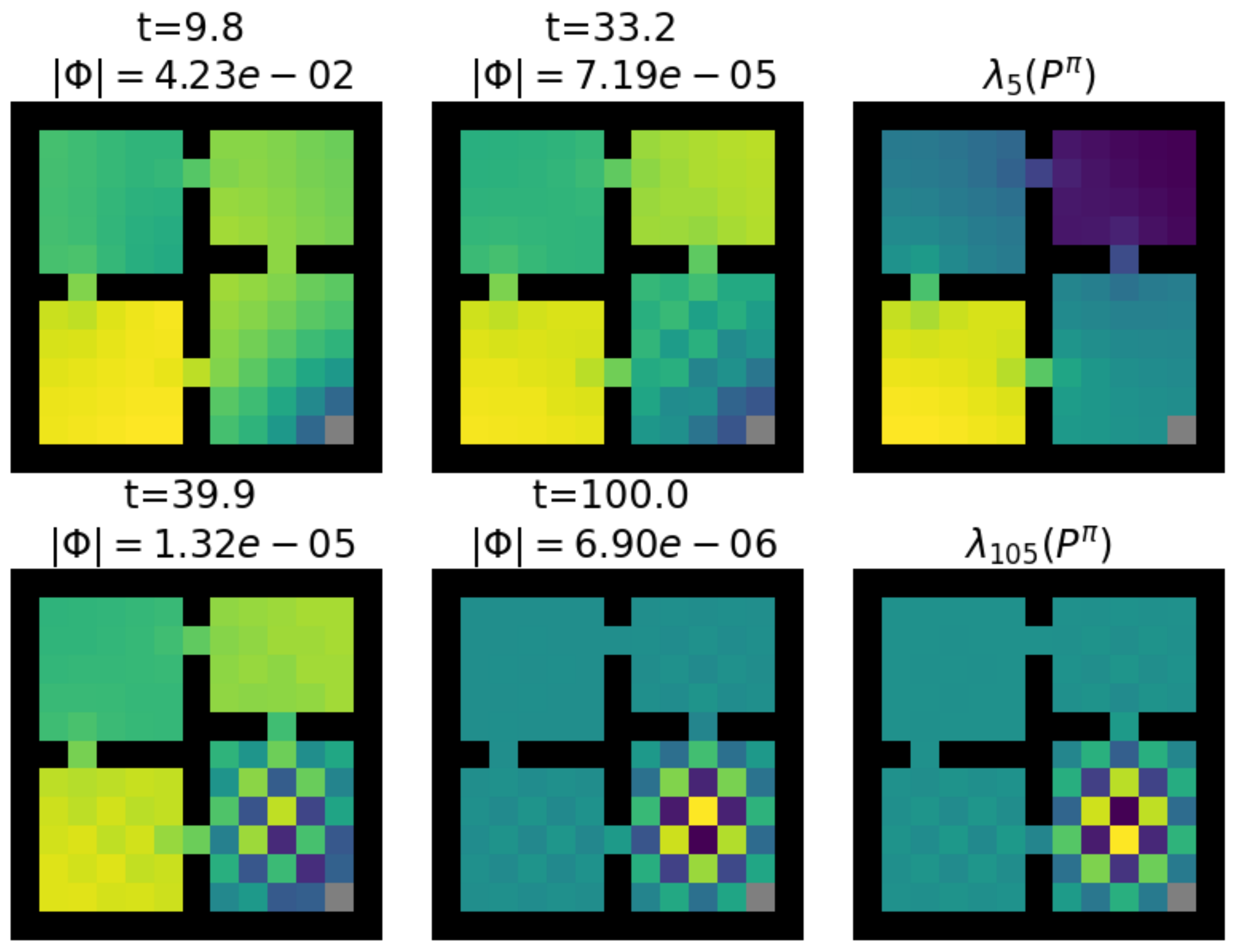}
    \caption{Visualization of a single column of $\Phi_t$ (i.e. feature vector)  after application of the ODE in Equation~\eqref{eq:phi-flow} for $t\in[0,100]$ in the four rooms environment, with $\repdim=10$ and $M = 20$. Early in its trajectory, $\phi_t$ exhibits similarity to smooth eigenfunctions (e.g. the eigenfunction corresponding to the 5$^{th}$ greatest eigenvalue $\lambda_5$ which we plot in the top right) of $P^\pi$, but later converges to non-smooth eigenfunctions (e.g. the eigenfunction corresponding to eigenvalue $\lambda_{105}$, the most negative eigenvalue, plotted in the bottom right).
    }
    
    \label{fig:feature-viz}
\end{figure}
 
We visualize $\phi_t$ along with two illustrative eigenfunctions of the transition matrix $P^\pi$, corresponding to one positive and one negative eigenvalue. 
We observe that while the feature $\phi_t$ quickly evolves to resemble the smooth eigenfunction corresponding to the positive eigenvalue for small values of $t$, it later converges to the non-smooth eigenfunction corresponding to the most negative eigenvalue of the transition matrix $P^\pi$. While we leave further analysis to future work, this example hints at an intriguing relationship between the EBFs and the joint representation dynamics.

\subsection{Random cumulants}\label{sec:random-cumulants}

In the case of zero rewards, our previous results show that whilst from the perspective of subspaces the representation approaches the EBF subspace in Grassmann distance, in Euclidean distance the representation is approaching the zero matrix pointwise. This has important implications for the scenario of large-scale sparse-reward environments, in which the agent may not encounter rewards for long periods of time, and indicates that the agent's representation is at risk of collapsing in such cases.

Motivated by this analysis, we consider a means of alleviating this representation collapse, by learning value functions for \emph{randomly generated cumulants} \citep{osband2018randomized,dabney2020value}. Mathematically, the agent again makes many predictions from a common representation, with each prediction indexed by $m=1,\ldots,M$ attempting to learn the value function associated with a randomly drawn reward function $r^m \in \mathbb{R}^{\statespace}$ under the policy $\pi$. Thus, the agent's parameters are the representation $\Phi$ and a set of weights $\mathbf{w}^m$ for each prediction. The learning dynamics are then given by:
\begin{align}
    \partial_t \Phi^{M}_t  \label{eq:rc-phi-flow}
    \!=  &  \alpha \sum_{m=1}^M (r^m\!\! +\!\! \gamma P^\pi \Phi^{M}_t \mathbf{w}_t^{m}\!\! -\!\! \Phi^{M}_t \mathbf{w}_t^{m})  (\mathbf{w}_t^{m} )^\top  , \\
    \partial_t \mathbf{w}_t^{m} = & \beta (\Phi^{M}_t)^\top (r^m + \gamma P^\pi \Phi^M_t \mathbf{w}^{m}_t - \Phi_t \mathbf{w}^{m}_t ) \, . \label{eq:rc-w-flow}
\end{align}

The main result of this section is to show that, even in the absence of reward, the limiting distribution induced by random cumulant auxiliary tasks dynamics described in Equation~\eqref{eq:rc-phi-flow} is not the zero subspace.

\begin{restatable}{theorem}{ThmDistribution}\label{thm:distribution}
For fixed $M \in \mathbb{N}$, let the random rewards $(r^m)_{m=1}^M$ and weights $(\mathbf{w}^m)_{m=1}^M$ be as defined above, let $\alpha=1$, and consider the representation dynamics in Equation~\eqref{eq:rc-phi-flow}, with weights fixed throughout training ($\beta=0$). Let $\Sigma$ denote the covariance matrix of the random cumulant distribution. Then 
\begin{align}
    &\lim_{M \rightarrow \infty} \sum_{m=1}^M r^m   (\mathbf{w}^m)^\top \overset{D}{=} Z_\Sigma \sim \mathcal{N}(0, \Sigma), \text{ and } \nonumber \\
   & \lim_{M \rightarrow \infty} \Phi^M_t \overset{D}{=}  \exp(-t(I - \gamma P^\pi ))(\Phi_0 - (I - \gamma P^\pi)^{-1} Z_\Sigma) \nonumber \\
    & \qquad\qquad+ (I - \gamma P^\pi)^{-1} Z_\Sigma \nonumber \;.
\end{align}

As the columns of $Z_\Sigma$ are mean-zero, uncorrelated, with covariance matrices $\Sigma$, the limiting distribution of each column of $\Phi_\infty = \lim_{t \rightarrow \infty} \lim_{M \rightarrow \infty} \Phi^M_t$ has covariance $\Psi\Sigma \Psi^\top$, where $\Psi$ is the resolvent $(I - \gamma P^\pi)^{-1}$.
\end{restatable}

\begin{restatable}{corollary}{propSubspaceConvergenceRC}\label{prop:subspaceconvergence-rc}
Under the feature flow \eqref{eq:rc-phi-flow} with $\mathbf{w}^m_t$ fixed at initialization for each $i = 1, \dots, M$ and Assumption~\ref{assume:value-function-conditions}, for almost all initialisations $\Phi_0$, we have when $R^\pi = 0$
    \begin{align*}
        d(\lim_{M \rightarrow \infty} \langle \Phi^M_t - \Phi_\infty \rangle, \langle U_{1:\repdim} \rangle) \rightarrow 0 \, ,\quad\text{as } t \rightarrow \infty.
    \end{align*}
\end{restatable}

Theorem~\ref{thm:distribution} indicates that the left-singular vectors of $\Sigma^{1/2}\Psi$ (or equivalently, the right-eigenvectors of $\Psi \Sigma \Psi^\top$) are key to understanding the effects on random cumulants on representations; we introduce the term \emph{resolvent singular basis functions} (RSBFs) to refer to these vectors in the canonical case $\Sigma = I$. 
\vspace{2mm}

\mdfsetup{%
backgroundcolor=black!10,
roundcorner=10pt}
\begin{mdframed}
\textbf{Key insight.}\\
With random cumulant auxiliary tasks, under the assumptions of Theorem~\ref{thm:distribution} and Corollary~\ref{prop:subspaceconvergence-rc}, the distribution of the limiting representation does not collapse, and is characterized by the RSBFs of $P^\pi$, while the trajectory it follows to reach this subspace is determined by the EBFs of $P^\pi$. 
\end{mdframed}

These decompositions of $P^\pi$ bear deep connections to prior work on feature learning. EBFs correspond to the eigendecomposition of the successor representation, which can be explicitly related to the proto-value functions described by \citet{mahadevan2009learning} when the transition matrix $P^\pi$ corresponds to that of a random walk policy \citep{machado2017eigenoption}. For symmetric $P^\pi$ we obtain an additional correspondence between EBFs and RSBFs, though we note that when $P^\pi$ is not symmetric the RSBFs may differ from both the EBFs and the singular value decomposition of the transition matrix $P^\pi$. We provide further discussion of RSBFs and comparisons against existing concepts in feature selection in Appendix~\ref{sec:feature-selection}.

In Appendix~\ref{sec:bayes-opt} we show that RSBFs can be viewed as Bayes-optimal features in the sense that they minimize the expected value function approximation error given an isotropic Gaussian prior on an unknown reward function.

\subsection{Analysis of additional auxiliary tasks}

The infinite-task limit simplifies the analysis of a broad range of auxiliary tasks, and analogous results to Theorem \ref{thm:infinite-heads} can be easily derived for families of auxiliary tasks which predict returns associated with additional policies and multiple discounts factors. We provide a summary of these results in Table \ref{tab:theory}, including their full statements and derivations in Appendix~\ref{sec:proofs}. 

We consider two additional classes of auxiliary task: predicting the values of multiple policies \citep{dabney2020value}, and predicting value under multiple discount factors \citep{fedus2019hyperbolic}. Under the multiple policies auxiliary task, the agent's objective is to learn a set of value functions $V^1, \dots, V^M$ such that $V^i(x) = \mathbb{E}_{\pi_i}[R^{\pi_i}(x) + \gamma P^{\pi_i}V^i(x)]$. Under the multiple discount factors auxiliary task, the agent's objective is analogously to find $V^i(x) = \mathbb{E}_{\pi_i}[R^{\pi}(x) + \gamma_i P^{\pi}V^i(x)]$ for $\gamma_i \in \gamma_1, \dots, \gamma_k$. We consider an ensemble prediction variant of these objectives, where given a fixed set of $k$ policies, we train an ensemble of $M$ predictors $V^{1,1}, \dots, V^{m, 1}, \dots, V^{1,k}, \dots, V^{m, k}$, where $m=\frac{M}{k}$ and the value function $V^{i, j}$ is trained on policy (respectively discount factor) $\pi_j$ (respectively $\gamma_j)$. 

In both cases, under the conditions of the previous theorems the dynamics of the ensemble converge to the dynamics induced by the mean of the set of auxiliary tasks, implying the counter-intuitive result that training with multiple auxiliary tasks doesn't provide additional utility over the single task setting. This apparent shortcoming can be addressed by ensuring that the weights corresponding to each auxiliary task $\pi_i$ or $\gamma_i$ are initialized in \textit{orthogonal subspaces}, so that the vector space $V$ in which the representation evolves can be decomposed as $V = \oplus_{i\in[1, k]} V_i$. In this case, we obtain an analogous decomposition of the representation $\Phi$ and its corresponding dynamics, obtaining convergence to a direct sum of the limiting representation of each task. This suggests that the benefits of auxiliary tasks might be maximized by appropriate initialization schemes which encourage the representations learned for each task to be independent.

\section{Experiments}\label{sec:experiments}

In this section, we complement the theoretical results above with empirical investigations in both tabular and deep reinforcement learning settings.

\subsection{Feature generalisation across the value-improvement path}\label{sec:feature-generalisation}

Having established connections between the representations induced by auxiliary tasks and several decompositions of the environment transition operator, we now turn to the question of how useful these representations are to a reinforcement learning agent. In particular, we address how well representations learnt under one policy generalize under the policy improvement step to approximate future value functions, with particular attention paid to EBFs and RSBFs, the decompositions that feature in our earlier analysis.

\begin{figure}[!b]
    \centering
    \includegraphics[keepaspectratio,width=.48\textwidth]{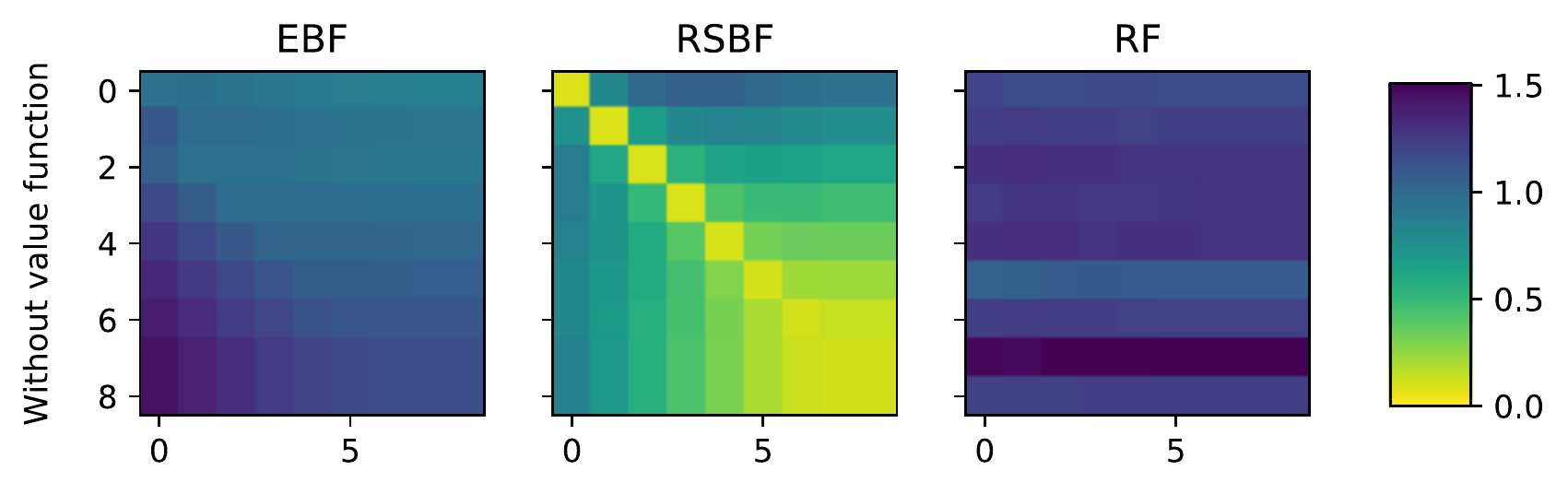}
    
    \includegraphics[keepaspectratio,width=.48\textwidth]{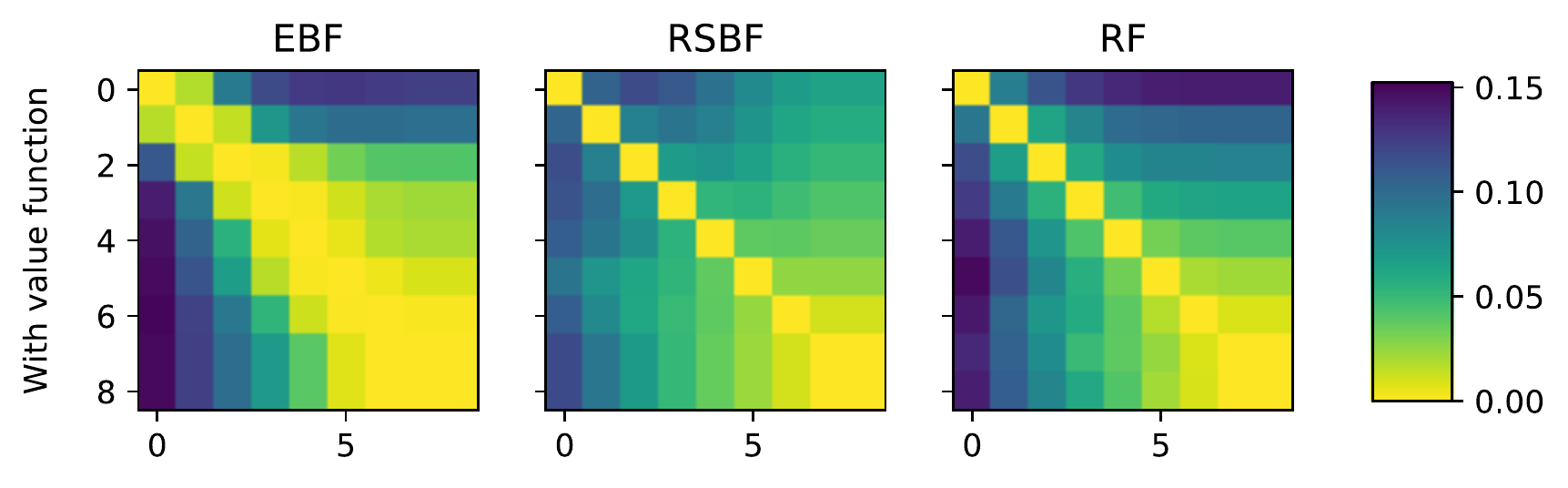}
    \caption{Transfer of EBFs, RSBFs, and RFs across the value-improvement path of a chain MDP, with and without the value function as an additional feature.
    }
    \label{fig:chain-transfer}
    \vspace{0.2cm}
\end{figure}

\begin{figure*}
    \includegraphics[width=\textwidth]{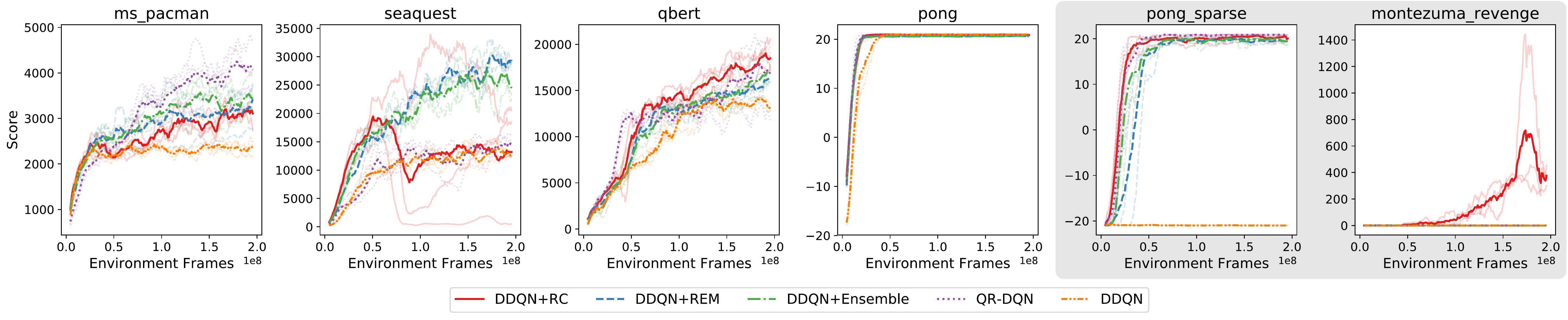}
    \caption{Learning curves for DDQN, DDQN+RC, DDQN+Ensemble, DDQN+REM, and QR-DQN agents on several dense reward ALE environments. Two games with sparse rewards are shown in the shaded box.}
    \label{fig:deep-rl}
    \vspace{0.5cm}
\end{figure*}
\begin{figure}
    \centering
    \includegraphics[width=.5\textwidth]{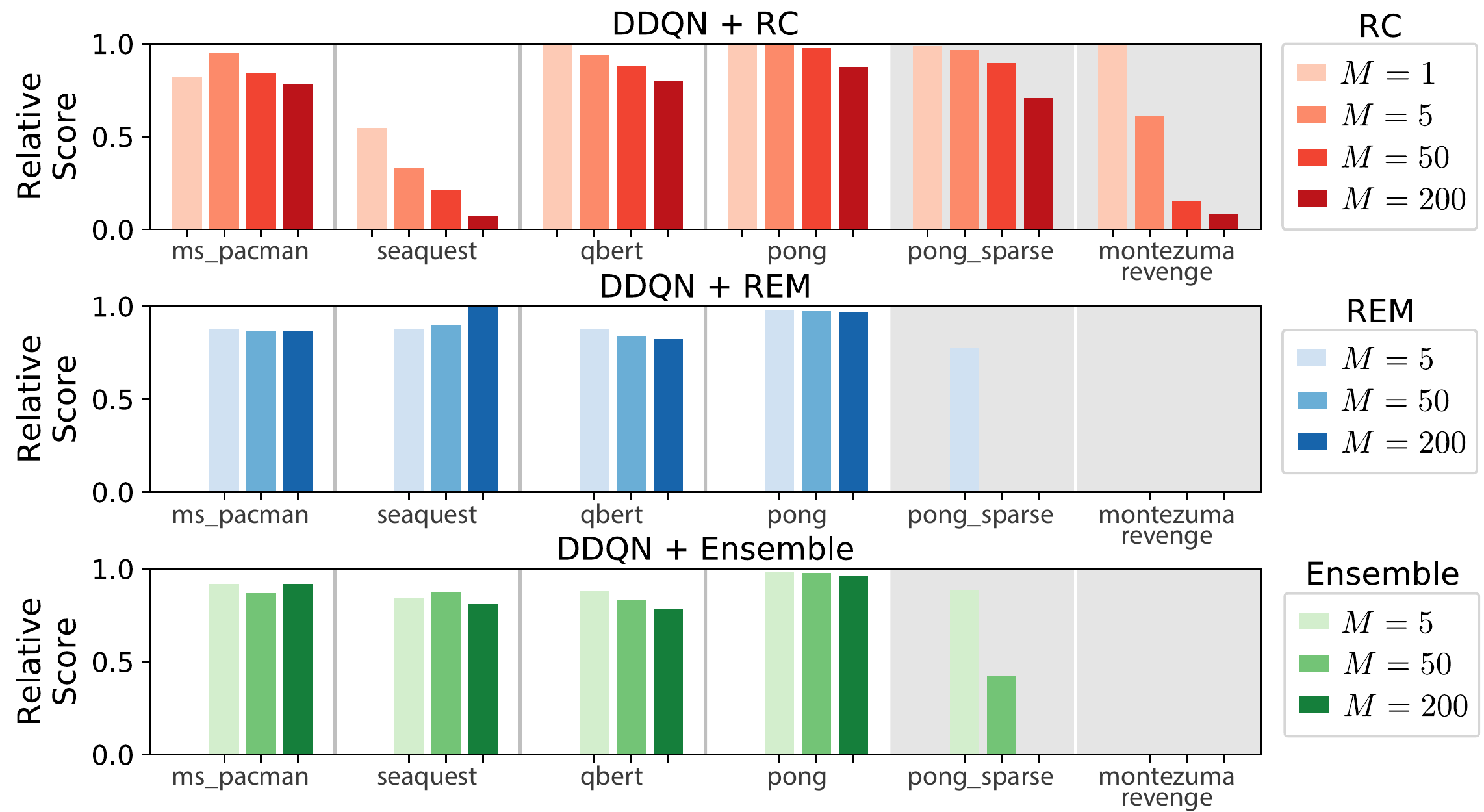}
    \caption{Sweep over number of auxiliary heads for RC, REM and Ensemble. Each bar corresponds to a single game, presented in the same order as Figure~\ref{fig:deep-rl}. Relative score is the per-game score divided by the maximum over all algorithms and hyperparameters.}
    \label{fig:naux}
\end{figure}

To address this question empirically we run tabular policy iteration on a stochastic chain MDP, yielding a sequence of policies $(\pi_j)_{j=1}^J$ and associated value functions $(V_j)_{j=1}^J$. We then compute EBFs and RSBFs associated with $P^\pi$, and compute the acute angle between $V_j$ and the subspace spanned by these features, for each $j \in [J]$; this is in fact equal to a generalisation of the Grassmann distance for subspaces of unequal dimension \citep{ye2016schubert}. We also compare against a baseline of isotropic randomly-generated features; full experimental details are provided in Appendix~\ref{sec:experiment-details}. 

Results are given in the top row of Figure~\ref{fig:chain-transfer} for the case of four features; each individual heatmap plots Grassmann distances, with rows indexing the policy that generated the features, and columns indexing the policy yielding the target value function. In general, the RSBFs provide better transfer across policies in the improvement path relative to random features and EBFs. For times $j, j' \in [J]$, we observe that  the Grassmann distance between the RSBFs of $P^{\pi_j}$ and the value function of $j'$, $V^{\pi_{j'}}$, increases as $|j - j'|$ does. 

We also evaluate transfer when the vector $V^{\pi_j}$ is added to the set of features, in the bottom row of Figure~\ref{fig:chain-transfer}. This contains the subspace to which the value functions described in Proposition~\ref{prop:many-value-functions} converge, as the limiting solutions can be described as being of the form $V^{\pi_j} + u$ for $u \in \langle U_{1:\repdim} \rangle$. Surprisingly, we find that in this setting the EBFs for $\pi_j$ outperform RSBFs specifically in predicting $V^{\pi_{j+1}}$. This can be observed in the upper off-diagonal the EBF plot in Figure~\ref{fig:chain-transfer}. We conclude that the dynamics induced by TD updates may be particularly beneficial to transfer between policies in the value-improvement path, and further study of this phenomenon is a promising avenue for future work.

\subsection{Auxiliary tasks for large-scale environments with sparse rewards}\label{sec:deep-rl}

We now consider the problem of deep RL agents interacting with environments with sparse reward structure. Motivated by the theoretical results obtained in earlier sections, we study the effects of a variety of auxiliary tasks in this setting; our analysis indicates that random cumulants may be particularly effective in preventing representation collapse in such environments.

We modify a Double DQN agent \citep{van2016deep} with a variety of auxiliary tasks, including random cumulants (RC) \citep{dabney2020value}, random ensemble mixtures (REM) \citep{agarwal2019striving}, an ensembling approach \citep{anschel2017averaged}, and also compare with QR-DQN, a distributional agent \citep{dabney2018distributional}. Full details of these agents, including specific implementation details for deep RL versions of these auxiliary tasks, are given in Appendix~\ref{sec:experiment-details}.

We evaluate these agents on a series of Atari games from the Arcade Learning Environment \citep{bellemare2013arcade,machado2018revisiting}, comprising Montezuma's Revenge, Pong, MsPacman, Seaquest, and Q$^*$bert.
In addition, we evaluate on a more challenging, sparse reward, version of Pong in which the agent does not observe negative rewards.\footnote{We attempted a similar modification of the other three dense reward games, but found no agent or configuration that was able to successfully learn on them. Full details, along with hyperparameters and results on these unsuccessful modifications, are given in Appendix~\ref{sec:experiment-details}.}

Figure~\ref{fig:deep-rl} shows the main results from these experiments.
Recall from Section~\ref{sec:aux-dynamics} that the random cumulant auxiliary task causes the agent's representation to converge to the RSBFs of $P^\pi$  in the idealized setting. 
We hypothesize that this auxiliary task will therefore improve agent performance over ensemble-based auxiliary tasks in sparse-reward environments.
Our empirical results support our hypothesis, with the random cumulant agent (DDQN+RC) generally performing well in the sparse-reward environments. Of particular note is the strong performance in Montezuma's Revenge. We expected reduced performance for DDQN+RC in the dense-reward games, but were surprised to observe improved performance here as well. However, we do note the instability seen in Seaquest.
Finally, Figure~\ref{fig:naux} shows the result of a hyperparameter sweep over the number of auxiliary task heads, revealing relevant differences in the three methods considered. Overall, we find that random cumulants are a promising auxiliary task specifically in sparse-reward environments, and believe that this motivates further theoretical development to close the gap between the dynamics of representations in deep RL agents, and the settings studied above.

\section{Related work}\label{sec:related-work}

As described previously, a wide variety of auxiliary tasks have been demonstrated to improve performance in deep reinforcement learning \citep{sutton2011horde, anschel2017averaged, jaderberg2017reinforcement, bellemare2017distributional,  barreto2017successor, mirowski2017learning, du2018adapting, riedmiller2018learning, oord2018representation, dabney2018distributional, gelada2019deepmdp, fedus2019hyperbolic,kartal2019terminal,lin2019adaptive, stooke2020decoupling,dabney2020value,guo2020bootstrap,laskin2020curl}. 
These works principally focus on demonstrating the empirical benefits of these tasks on agent performance, leaving an analysis as to why these effects occur to future work.
Follow-up work on distributional reinforcement learning, for example, has begun to close the theory-practice gap \citep{lyle2019comparative, rowland2018analysis}. There is also a growing body of work on understanding the impact of representations on the sample efficiency of reinforcement learning; see for example \citet{du2019good,van2019comments,lattimore2020learning}.

Further analysis of auxiliary tasks in deep reinforcement learning focuses on their effect on the representation learned by the agent \citep{bellemare2019geometric, dabney2020value} and its ability to approximate the value functions of several policies. Additionally, \citet{ghosh2020representations} propose an auxiliary task based on its effect on the \textit{stability} of learned representations. \citet{kumar2020implicit} also study representation collapse in deep reinforcement learning, in the absence of auxiliary tasks. Aside from reinforcement learning, there are also related empirical approaches using bootstrapping to shape representations in self-supervised learning \citep{grill2020bootstrap}, and theoretical work in characterising the regularising effect of self-distillation \citep{mobahi2020self} and over-parametrisation \citep{arora2019implicit,li2018algorithmic} in supervised learning.

Recent work in representation learning has its roots in the broader feature selection problem in reinforcement learning. This problem has been extensively studied \citep{parr2008analysis,parr2007analyzing,mahadevan2007proto,petrik2007analysis,mahadevan2009learning,kroon2009automatic,fard2013bellman,jiang2015abstraction},
particularly in the linear value function approximation setting.
\citet{parr2008analysis} analyze power-iteration-style feature learning methods, to which our analysis of the convergence of features presents notable similarity. Also closely related is the notion of feature adaptation \citep{menache2005basis,yu2009basis,di2010adaptive,bhatnagar2013feature,prabuchandran2014actor,prabuchandran2016actor,barker2019unsupervised}, in which features are adaptively updated simultaneously with the weights used for value function approximation.

\section{Conclusion}

We have introduced a framework based on learning dynamics to analyse representations in reinforcement learning. This led to a variety of theoretical results concerning learning with and without the presence of auxiliary tasks, as well as several straightforward models for studying representation learning empirically. With this, we were able to thoroughly test a new hypothesis on the effectiveness of particular auxiliary tasks in sparse reward environments, which led to improved understanding of representation learning in RL, as well as practical modifications to deep RL algorithms.

There are many natural follow-up directions to this work. One direction is to further develop the theory associated with the learning dynamics perspective, in order to (i) understand how additional types of auxiliary tasks, in particular auxiliary tasks that don't correspond to value functions, affect the representations in the learning models developed in this paper, (ii) extend the learning models themselves to incorporate further aspects of large-scale learning scenarios, such as sample-based learning and state-visitation distribution corrections, and (iii) investigate other common learning dynamics, such as gradient TD methods \citep{sutton2008convergent}.
There is also scope for further empirical work to develop an understanding of which auxiliary tasks are useful in certain types of environments, extending the observations relating to sparse-reward environments and random cumulants in this paper.
We hope that the community will find this framework useful for testing many more hypotheses in a wide range of scenarios, ultimately leading to a better understanding of how reinforcement learning and representation learning fit together.

\clearpage

\section*{Acknowledgements}

We thank Diana Borsa for detailed feedback on a preliminary version of this paper, and Kris Cao, Marc Bellemare, and the wider DeepMind team for valuable discussions. We also thank the anonymous reviewers for useful comments during the review process. CL is funded by an Open Phil AI Fellowship.

\bibliographystyle{plainnat}
\bibliography{references}

\begin{thebibliography}{71}
\providecommand{\natexlab}[1]{#1}
\providecommand{\url}[1]{\texttt{#1}}
\expandafter\ifx\csname urlstyle\endcsname\relax
  \providecommand{\doi}[1]{doi: #1}\else
  \providecommand{\doi}{doi: \begingroup \urlstyle{rm}\Url}\fi

\bibitem[Agarwal et~al.(2020)Agarwal, Schuurmans, and
  Norouzi]{agarwal2019striving}
Rishabh Agarwal, Dale Schuurmans, and Mohammad Norouzi.
\newblock An optimistic perspective on offline reinforcement learning.
\newblock In \emph{International Conference on Machine Learning (ICML)}, 2020.

\bibitem[Anschel et~al.(2017)Anschel, Baram, and Shimkin]{anschel2017averaged}
Oron Anschel, Nir Baram, and Nahum Shimkin.
\newblock Averaged-{DQN}: Variance reduction and stabilization for deep
  reinforcement learning.
\newblock In \emph{International Conference on Machine Learning (ICML)}, 2017.

\bibitem[Arora et~al.(2019{\natexlab{a}})Arora, Cohen, Hu, and
  Luo]{arora2019implicit}
Sanjeev Arora, Nadav Cohen, Wei Hu, and Yuping Luo.
\newblock Implicit regularization in deep matrix factorization.
\newblock In \emph{Neural Information Processing Systems (NeurIPS)},
  2019{\natexlab{a}}.

\bibitem[Arora et~al.(2019{\natexlab{b}})Arora, Du, Hu, Li, and
  Wang]{arora2019fine}
Sanjeev Arora, Simon~S Du, Wei Hu, Zhiyuan Li, and Ruosong Wang.
\newblock Fine-grained analysis of optimization and generalization for
  overparameterized two-layer neural networks.
\newblock In \emph{International Conference on Machine Learning (ICML)},
  2019{\natexlab{b}}.

\bibitem[Barker and Ras(2019)]{barker2019unsupervised}
Edward Barker and Charl Ras.
\newblock Unsupervised basis function adaptation for reinforcement learning.
\newblock \emph{Journal of Machine Learning Research}, 20\penalty0
  (128):\penalty0 1--73, 2019.

\bibitem[Barreto et~al.(2017)Barreto, Dabney, Munos, Hunt, Schaul, van Hasselt,
  and Silver]{barreto2017successor}
Andr{\'e} Barreto, Will Dabney, R{\'e}mi Munos, Jonathan~J Hunt, Tom Schaul,
  Hado van Hasselt, and David Silver.
\newblock Successor features for transfer in reinforcement learning.
\newblock In \emph{Neural Information Processing Systems (NeurIPS)}, 2017.

\bibitem[Behzadian and Petrik(2018)]{behzadian2018feature}
Bahram Behzadian and Marek Petrik.
\newblock Feature selection by singular value decomposition for reinforcement
  learning.
\newblock In \emph{ICML Prediction and Generative Modeling Workshop}, 2018.

\bibitem[Bellemare et~al.(2013)Bellemare, Naddaf, Veness, and
  Bowling]{bellemare2013arcade}
Marc~G Bellemare, Yavar Naddaf, Joel Veness, and Michael Bowling.
\newblock The {A}rcade {L}earning {E}nvironment: An evaluation platform for
  general agents.
\newblock \emph{Journal of Artificial Intelligence Research}, 47:\penalty0
  253--279, 2013.

\bibitem[Bellemare et~al.(2017)Bellemare, Dabney, and
  Munos]{bellemare2017distributional}
Marc~G Bellemare, Will Dabney, and R{\'e}mi Munos.
\newblock A distributional perspective on reinforcement learning.
\newblock In \emph{International Conference on Machine Learning (ICML)}, 2017.

\bibitem[Bellemare et~al.(2019)Bellemare, Dabney, Dadashi, Ta{\"i}ga, Castro,
  Le~Roux, Schuurmans, Lattimore, and Lyle]{bellemare2019geometric}
Marc~G Bellemare, Will Dabney, Robert Dadashi, Adrien~Ali Ta{\"i}ga,
  Pablo~Samuel Castro, Nicolas Le~Roux, Dale Schuurmans, Tor Lattimore, and
  Clare Lyle.
\newblock A geometric perspective on optimal representations for reinforcement
  learning.
\newblock In \emph{Neural Information Processing Systems (NeurIPS)}, 2019.

\bibitem[Bertsekas(2018)]{bertsekas2018feature}
Dimitri~P Bertsekas.
\newblock Feature-based aggregation and deep reinforcement learning: A survey
  and some new implementations.
\newblock \emph{IEEE/CAA Journal of Automatica Sinica}, 6\penalty0
  (1):\penalty0 1--31, 2018.

\bibitem[Bertsekas and Tsitsiklis(1996)]{bertsekas1996neuro}
Dimitri~P Bertsekas and John~N Tsitsiklis.
\newblock \emph{Neuro-Dynamic Programming}.
\newblock Athena Scientific, 1st edition, 1996.

\bibitem[Bhatnagar et~al.(2013)Bhatnagar, Borkar, and
  Prabuchandran]{bhatnagar2013feature}
Shalabh Bhatnagar, Vivek~S Borkar, and KJ~Prabuchandran.
\newblock Feature search in the {G}rassmanian in online reinforcement learning.
\newblock \emph{IEEE Journal of Selected Topics in Signal Processing},
  7\penalty0 (5):\penalty0 746--758, 2013.

\bibitem[Boyan(1999)]{boyan1999least}
Justin~A Boyan.
\newblock Least-squares temporal difference learning.
\newblock In \emph{ICML}, pages 49--56. Citeseer, 1999.

\bibitem[Bradbury et~al.(2018)Bradbury, Frostig, Hawkins, Johnson, Leary,
  Maclaurin, Necula, Paszke, Vander{P}las, Wanderman-{M}ilne, and
  Zhang]{jax2018github}
James Bradbury, Roy Frostig, Peter Hawkins, Matthew~James Johnson, Chris Leary,
  Dougal Maclaurin, George Necula, Adam Paszke, Jake Vander{P}las, Skye
  Wanderman-{M}ilne, and Qiao Zhang.
\newblock {JAX}: composable transformations of {P}ython+{N}um{P}y programs,
  2018.

\bibitem[Chung et~al.(2018)Chung, Nath, Joseph, and White]{chung2018two}
Wesley Chung, Somjit Nath, Ajin Joseph, and Martha White.
\newblock Two-timescale networks for nonlinear value function approximation.
\newblock In \emph{International Conference on Learning Representations
  (ICLR)}, 2018.

\bibitem[Dabney et~al.(2018)Dabney, Rowland, Bellemare, and
  Munos]{dabney2018distributional}
Will Dabney, Mark Rowland, Marc~G Bellemare, and R{\'e}mi Munos.
\newblock Distributional reinforcement learning with quantile regression.
\newblock In \emph{AAAI Conference on Artificial Intelligence}, 2018.

\bibitem[Dabney et~al.(2020)Dabney, Barreto, Rowland, Dadashi, Quan, Bellemare,
  and Silver]{dabney2020value}
Will Dabney, Andr{\'e} Barreto, Mark Rowland, Robert Dadashi, John Quan, Marc~G
  Bellemare, and David Silver.
\newblock The value-improvement path: Towards better representations for
  reinforcement learning.
\newblock \emph{arXiv}, 2020.

\bibitem[Di~Castro and Mannor(2010)]{di2010adaptive}
Dotan Di~Castro and Shie Mannor.
\newblock Adaptive bases for reinforcement learning.
\newblock In \emph{Joint European Conference on Machine Learning and Knowledge
  Discovery in Databases}, pages 312--327. Springer, 2010.

\bibitem[Du et~al.(2019)Du, Kakade, Wang, and Yang]{du2019good}
Simon~S Du, Sham~M Kakade, Ruosong Wang, and Lin~F Yang.
\newblock Is a good representation sufficient for sample efficient
  reinforcement learning?
\newblock In \emph{International Conference on Learning Representations
  (ICLR)}, 2019.

\bibitem[Du et~al.(2018)Du, Czarnecki, Jayakumar, Pascanu, and
  Lakshminarayanan]{du2018adapting}
Yunshu Du, Wojciech~M Czarnecki, Siddhant~M Jayakumar, Razvan Pascanu, and
  Balaji Lakshminarayanan.
\newblock Adapting auxiliary losses using gradient similarity.
\newblock \emph{arXiv}, 2018.

\bibitem[Fard et~al.(2013)Fard, Grinberg, Farahmand, Pineau, and
  Precup]{fard2013bellman}
Mahdi~Milani Fard, Yuri Grinberg, Amir-massoud Farahmand, Joelle Pineau, and
  Doina Precup.
\newblock Bellman error based feature generation using random projections on
  sparse spaces.
\newblock In \emph{Neural Information Processing Systems (NIPS)}, 2013.

\bibitem[Fedus et~al.(2019)Fedus, Gelada, Bengio, Bellemare, and
  Larochelle]{fedus2019hyperbolic}
William Fedus, Carles Gelada, Yoshua Bengio, Marc~G Bellemare, and Hugo
  Larochelle.
\newblock Hyperbolic discounting and learning over multiple horizons.
\newblock In \emph{Reinforcement Learning and Decision Making (RLDM)}, 2019.

\bibitem[Gelada et~al.(2019)Gelada, Kumar, Buckman, Nachum, and
  Bellemare]{gelada2019deepmdp}
Carles Gelada, Saurabh Kumar, Jacob Buckman, Ofir Nachum, and Marc~G Bellemare.
\newblock Deep{MDP}: Learning continuous latent space models for representation
  learning.
\newblock In \emph{International Conference on Machine Learning (ICML)}, 2019.

\bibitem[Ghosh and Bellemare(2020)]{ghosh2020representations}
Dibya Ghosh and Marc~G. Bellemare.
\newblock Representations for stable off-policy reinforcement learning.
\newblock In \emph{International Conference on Machine Learning (ICML)}, 2020.

\bibitem[Grill et~al.(2020)Grill, Strub, Altch{\'e}, Tallec, Richemond,
  Buchatskaya, Doersch, Pires, Guo, Azar, Piot, Kavukcuoglu, Munos, and
  Valko]{grill2020bootstrap}
Jean-Bastien Grill, Florian Strub, Florent Altch{\'e}, Corentin Tallec,
  Pierre~H Richemond, Elena Buchatskaya, Carl Doersch, Bernardo~Avila Pires,
  Zhaohan~Daniel Guo, Mohammad~Gheshlaghi Azar, Bilal Piot, Koray Kavukcuoglu,
  R\'emi Munos, and Michal Valko.
\newblock Bootstrap your own latent: {A} new approach to self-supervised
  learning.
\newblock \emph{arXiv}, 2020.

\bibitem[Guo et~al.(2020)Guo, Pires, Piot, Grill, Altch{\'e}, Munos, and
  Azar]{guo2020bootstrap}
Daniel Guo, Bernardo~Avila Pires, Bilal Piot, Jean-bastien Grill, Florent
  Altch{\'e}, R{\'e}mi Munos, and Mohammad~Gheshlaghi Azar.
\newblock Bootstrap latent-predictive representations for multitask
  reinforcement learning.
\newblock \emph{arXiv}, 2020.

\bibitem[Jaakkola et~al.(1994)Jaakkola, Jordan, and
  Singh]{jaakola1994convergence}
Tommi Jaakkola, Michael~I. Jordan, and Satinder~P Singh.
\newblock On the convergence of stochastic iterative dynamic programming
  algorithms.
\newblock \emph{Neural Computation}, 6\penalty0 (6), 1994.

\bibitem[Jaderberg et~al.(2017{\natexlab{a}})Jaderberg, Mnih, Czarnecki,
  Schaul, Leibo, Silver, and Kavukcuoglu]{jaderberg2016reinforcement}
Max Jaderberg, Volodymyr Mnih, Wojciech~Marian Czarnecki, Tom Schaul, Joel~Z
  Leibo, David Silver, and Koray Kavukcuoglu.
\newblock Reinforcement learning with unsupervised auxiliary tasks.
\newblock In \emph{International Conference on Learning Representations
  (ICLR)}, 2017{\natexlab{a}}.

\bibitem[Jaderberg et~al.(2017{\natexlab{b}})Jaderberg, Mnih, Czarnecki,
  Schaul, Leibo, Silver, and Kavukcuoglu]{jaderberg2017reinforcement}
Max Jaderberg, Volodymyr Mnih, Wojciech~Marian Czarnecki, Tom Schaul, Joel~Z
  Leibo, David Silver, and Koray Kavukcuoglu.
\newblock Reinforcement learning with unsupervised auxiliary tasks.
\newblock In \emph{International Conference on Learning Representations
  (ICLR)}, 2017{\natexlab{b}}.

\bibitem[Jiang et~al.(2015)Jiang, Kulesza, and Singh]{jiang2015abstraction}
Nan Jiang, Alex Kulesza, and Satinder Singh.
\newblock Abstraction selection in model-based reinforcement learning.
\newblock In \emph{International Conference on Machine Learning (ICML)}, 2015.

\bibitem[Kartal et~al.(2019)Kartal, Hernandez-Leal, and
  Taylor]{kartal2019terminal}
Bilal Kartal, Pablo Hernandez-Leal, and Matthew~E Taylor.
\newblock Terminal prediction as an auxiliary task for deep reinforcement
  learning.
\newblock In \emph{AAAI Conference on Artificial Intelligence and Interactive
  Digital Entertainment}, 2019.

\bibitem[Kroon and Whiteson(2009)]{kroon2009automatic}
Mark Kroon and Shimon Whiteson.
\newblock Automatic feature selection for model-based reinforcement learning in
  factored {MDP}s.
\newblock In \emph{2009 International Conference on Machine Learning and
  Applications}, pages 324--330. IEEE, 2009.

\bibitem[Kumar et~al.(2021)Kumar, Agarwal, Ghosh, and
  Levine]{kumar2020implicit}
Aviral Kumar, Rishabh Agarwal, Dibya Ghosh, and Sergey Levine.
\newblock Implicit under-parameterization inhibits data-efficient deep
  reinforcement learning.
\newblock In \emph{International Conference on Learning Representations
  (ICLR)}, 2021.

\bibitem[Laskin et~al.(2020)Laskin, Srinivas, and Abbeel]{laskin2020curl}
Michael Laskin, Aravind Srinivas, and Pieter Abbeel.
\newblock {CURL}: Contrastive unsupervised representations for reinforcement
  learning.
\newblock In \emph{International Conference on Machine Learning (ICML)}, 2020.

\bibitem[Lattimore et~al.(2020)Lattimore, Szepesvari, and
  Weisz]{lattimore2020learning}
Tor Lattimore, Csaba Szepesvari, and Gellert Weisz.
\newblock Learning with good feature representations in bandits and in {RL}
  with a generative model.
\newblock In \emph{International Conference on Machine Learning (ICML)}, 2020.

\bibitem[Levine et~al.(2017)Levine, Zahavy, Mankowitz, Tamar, and
  Mannor]{levine2017shallow}
Nir Levine, Tom Zahavy, Daniel~J Mankowitz, Aviv Tamar, and Shie Mannor.
\newblock Shallow updates for deep reinforcement learning.
\newblock In \emph{Neural Information Processing Systems (NeurIPS)}, 2017.

\bibitem[Li et~al.(2018)Li, Ma, and Zhang]{li2018algorithmic}
Yuanzhi Li, Tengyu Ma, and Hongyang Zhang.
\newblock Algorithmic regularization in over-parameterized matrix sensing and
  neural networks with quadratic activations.
\newblock In \emph{Conference On Learning Theory (COLT)}, 2018.

\bibitem[Lin et~al.(2019)Lin, Baweja, Kantor, and Held]{lin2019adaptive}
Xingyu Lin, Harjatin Baweja, George Kantor, and David Held.
\newblock Adaptive auxiliary task weighting for reinforcement learning.
\newblock In \emph{Neural Information Processing Systems (NeurIPS)}, 2019.

\bibitem[Lyle et~al.(2019)Lyle, Bellemare, and Castro]{lyle2019comparative}
Clare Lyle, Marc~G Bellemare, and Pablo~Samuel Castro.
\newblock A comparative analysis of expected and distributional reinforcement
  learning.
\newblock In \emph{AAAI Conference on Artificial Intelligence}, 2019.

\bibitem[Machado et~al.(2017)Machado, Bellemare, and
  Bowling]{machado2017laplacian}
Marios~C Machado, Marc~G Bellemare, and Michael Bowling.
\newblock A {L}aplacian framework for option discovery in reinforcement
  learning.
\newblock In \emph{International Conference on Machine Learning (ICML)}, 2017.

\bibitem[Machado et~al.(2018{\natexlab{a}})Machado, Bellemare, Talvitie,
  Veness, Hausknecht, and Bowling]{machado2018revisiting}
Marlos~C Machado, Marc~G Bellemare, Erik Talvitie, Joel Veness, Matthew
  Hausknecht, and Michael Bowling.
\newblock Revisiting the {A}rcade {L}earning {E}nvironment: Evaluation
  protocols and open problems for general agents.
\newblock \emph{Journal of Artificial Intelligence Research}, 61:\penalty0
  523--562, 2018{\natexlab{a}}.

\bibitem[Machado et~al.(2018{\natexlab{b}})Machado, Rosenbaum, Guo, Liu,
  Tesauro, and Campbell]{machado2017eigenoption}
Marlos~C Machado, Clemens Rosenbaum, Xiaoxiao Guo, Miao Liu, Gerald Tesauro,
  and Murray Campbell.
\newblock Eigenoption discovery through the deep successor representation.
\newblock In \emph{International Conference on Learning Representations
  (ICLR)}, 2018{\natexlab{b}}.

\bibitem[Mahadevan(2009)]{mahadevan2009learning}
Sridhar Mahadevan.
\newblock Learning representation and control in {M}arkov decision processes:
  {N}ew frontiers.
\newblock \emph{Foundations and Trends{\textregistered} in Machine Learning},
  1\penalty0 (4):\penalty0 403--565, 2009.

\bibitem[Mahadevan and Maggioni(2007)]{mahadevan2007proto}
Sridhar Mahadevan and Mauro Maggioni.
\newblock Proto-value functions: A {L}aplacian framework for learning
  representation and control in {M}arkov decision processes.
\newblock \emph{Journal of Machine Learning Research}, 8\penalty0
  (Oct):\penalty0 2169--2231, 2007.

\bibitem[Menache et~al.(2005)Menache, Mannor, and Shimkin]{menache2005basis}
Ishai Menache, Shie Mannor, and Nahum Shimkin.
\newblock Basis function adaptation in temporal difference reinforcement
  learning.
\newblock \emph{Annals of Operations Research}, 134\penalty0 (1):\penalty0
  215--238, 2005.

\bibitem[Mirowski et~al.(2017)Mirowski, Pascanu, Viola, Soyer, Ballard, Banino,
  Denil, Goroshin, Sifre, Kavukcuoglu, Kumaran, and
  Hadsell]{mirowski2017learning}
Piotr Mirowski, Razvan Pascanu, Fabio Viola, Hubert Soyer, Andrew~J Ballard,
  Andrea Banino, Misha Denil, Ross Goroshin, Laurent Sifre, Koray Kavukcuoglu,
  Dharshan Kumaran, and Raia Hadsell.
\newblock Learning to navigate in complex environments.
\newblock In \emph{International Conference on Learning Representations
  (ICLR)}, 2017.

\bibitem[Mnih et~al.(2015)Mnih, Kavukcuoglu, Silver, Rusu, Veness, Bellemare,
  Graves, Riedmiller, Fidjeland, Ostrovski, Petersen, Beattie, Sadik,
  Antonoglou, King, Kumaran, Wierstra, Legg, and Hassabis]{mnih2015human}
Volodymyr Mnih, Koray Kavukcuoglu, David Silver, Andrei~A Rusu, Joel Veness,
  Marc~G Bellemare, Alex Graves, Martin Riedmiller, Andreas~K Fidjeland, Georg
  Ostrovski, Stig Petersen, Charles Beattie, Amir Sadik, Ioannis Antonoglou,
  Helen King, Dharshan Kumaran, Daan Wierstra, Shane Legg, and Demis Hassabis.
\newblock Human-level control through deep reinforcement learning.
\newblock \emph{Nature}, 518\penalty0 (7540):\penalty0 529--533, 2015.

\bibitem[Mobahi et~al.(2020)Mobahi, Farajtabar, and Bartlett]{mobahi2020self}
Hossein Mobahi, Mehrdad Farajtabar, and Peter~L Bartlett.
\newblock Self-distillation amplifies regularization in {H}ilbert space.
\newblock In \emph{Neural Information Processing Systems (NeurIPS)}, 2020.

\bibitem[Osband et~al.(2016)Osband, Blundell, Pritzel, and
  Van~Roy]{osband2016deep}
Ian Osband, Charles Blundell, Alexander Pritzel, and Benjamin Van~Roy.
\newblock Deep exploration via bootstrapped {DQN}.
\newblock In \emph{Neural Information Processing Systems (NIPS)}, 2016.

\bibitem[Osband et~al.(2018)Osband, Aslanides, and
  Cassirer]{osband2018randomized}
Ian Osband, John Aslanides, and Albin Cassirer.
\newblock Randomized prior functions for deep reinforcement learning.
\newblock In \emph{Neural Information Processing Systems (NeurIPS)}, 2018.

\bibitem[Parr et~al.(2007)Parr, Painter-Wakefield, Li, and
  Littman]{parr2007analyzing}
Ronald Parr, Christopher Painter-Wakefield, Lihong Li, and Michael Littman.
\newblock Analyzing feature generation for value-function approximation.
\newblock In \emph{International Conference on Machine Learning (ICML)}, 2007.

\bibitem[Parr et~al.(2008)Parr, Li, Taylor, Painter-Wakefield, and
  Littman]{parr2008analysis}
Ronald Parr, Lihong Li, Gavin Taylor, Christopher Painter-Wakefield, and
  Michael~L Littman.
\newblock An analysis of linear models, linear value-function approximation,
  and feature selection for reinforcement learning.
\newblock In \emph{International Conference on Machine Learning (ICML)}, 2008.

\bibitem[Petrik(2007)]{petrik2007analysis}
Marek Petrik.
\newblock An analysis of {L}aplacian methods for value function approximation
  in {MDP}s.
\newblock In \emph{International Joint Conference on Artificial Intelligence
  (IJCAI)}, 2007.

\bibitem[Prabuchandran et~al.(2014)Prabuchandran, Bhatnagar, and
  Borkar]{prabuchandran2014actor}
KJ~Prabuchandran, Shalabh Bhatnagar, and Vivek~S Borkar.
\newblock An actor critic algorithm based on {G}rassmanian search.
\newblock In \emph{IEEE Conference on Decision and Control}, 2014.

\bibitem[Prabuchandran et~al.(2016)Prabuchandran, Bhatnagar, and
  Borkar]{prabuchandran2016actor}
KJ~Prabuchandran, Shalabh Bhatnagar, and Vivek~S Borkar.
\newblock Actor-critic algorithms with online feature adaptation.
\newblock \emph{ACM Transactions on Modeling and Computer Simulation (TOMACS)},
  26\penalty0 (4):\penalty0 1--26, 2016.

\bibitem[Quan and Ostrovski(2020)]{dqnzoo2020github}
John Quan and Georg Ostrovski.
\newblock {DQN} {Zoo}: Reference implementations of {DQN}-based agents, 2020.

\bibitem[Riedmiller et~al.(2018)Riedmiller, Hafner, Lampe, Neunert, Degrave,
  Wiele, Mnih, Heess, and Springenberg]{riedmiller2018learning}
Martin Riedmiller, Roland Hafner, Thomas Lampe, Michael Neunert, Jonas Degrave,
  Tom Wiele, Vlad Mnih, Nicolas Heess, and Jost~Tobias Springenberg.
\newblock Learning by playing - solving sparse reward tasks from scratch.
\newblock In \emph{International Conference on Machine Learning (ICML)}, 2018.

\bibitem[Rowland et~al.(2018)Rowland, Bellemare, Dabney, Munos, and
  Teh]{rowland2018analysis}
Mark Rowland, Marc~G Bellemare, Will Dabney, R{\'e}mi Munos, and Yee~Whye Teh.
\newblock An analysis of categorical distributional reinforcement learning.
\newblock In \emph{Artificial Intelligence and Statistics (AISTATS)}, 2018.

\bibitem[Stachenfeld et~al.(2014)Stachenfeld, Botvinick, and
  Gershman]{stachenfeld2014design}
Kimberly~L Stachenfeld, Matthew Botvinick, and Samuel~J Gershman.
\newblock Design principles of the hippocampal cognitive map.
\newblock In \emph{Neural Information Processing Systems (NIPS)}, 2014.

\bibitem[Stooke et~al.(2020)Stooke, Lee, Abbeel, and
  Laskin]{stooke2020decoupling}
Adam Stooke, Kimin Lee, Pieter Abbeel, and Michael Laskin.
\newblock Decoupling representation learning from reinforcement learning.
\newblock \emph{arXiv}, 2020.

\bibitem[Sutton et~al.(2008)Sutton, Szepesv{\'a}ri, and
  Maei]{sutton2008convergent}
Richard~S Sutton, Csaba Szepesv{\'a}ri, and Hamid~Reza Maei.
\newblock A convergent {$O(n)$} algorithm for off-policy temporal-difference
  learning with linear function approximation.
\newblock \emph{Neural Information Processing Systems (NIPS)}, 2008.

\bibitem[Sutton et~al.(2011)Sutton, Modayil, Delp, Degris, Pilarski, White, and
  Precup]{sutton2011horde}
Richard~S Sutton, Joseph Modayil, Michael Delp, Thomas Degris, Patrick~M
  Pilarski, Adam White, and Doina Precup.
\newblock Horde: A scalable real-time architecture for learning knowledge from
  unsupervised sensorimotor interaction.
\newblock In \emph{International Conference on Autonomous Agents and Multiagent
  Systems (AAMAS)}, 2011.

\bibitem[Tsitsiklis(1994)]{tsitsiklis1994asynchronous}
John~N Tsitsiklis.
\newblock Asynchronous stochastic approximation and {Q}-learning.
\newblock \emph{Machine learning}, 16\penalty0 (3):\penalty0 185--202, 1994.

\bibitem[van~den Oord et~al.(2018)van~den Oord, Li, and
  Vinyals]{oord2018representation}
A{\"a}ron van~den Oord, Yazhe Li, and Oriol Vinyals.
\newblock Representation learning with contrastive predictive coding.
\newblock \emph{arXiv}, 2018.

\bibitem[Van~Hasselt et~al.(2016)Van~Hasselt, Guez, and Silver]{van2016deep}
Hado Van~Hasselt, Arthur Guez, and David Silver.
\newblock Deep reinforcement learning with double {Q}-learning.
\newblock In \emph{AAAI Conference on Artificial Intelligence}, 2016.

\bibitem[Van~Roy and Dong(2019)]{van2019comments}
Benjamin Van~Roy and Shi Dong.
\newblock Comments on the {D}u-{K}akade-{W}ang-{Y}ang lower bounds.
\newblock \emph{arXiv}, 2019.

\bibitem[Virtanen et~al.(2020)Virtanen, Gommers, Oliphant, Haberland, Reddy,
  Cournapeau, Burovski, Peterson, Weckesser, Bright, {van der Walt}, Brett,
  Wilson, Millman, Mayorov, Nelson, Jones, Kern, Larson, Carey, Polat, Feng,
  Moore, {VanderPlas}, Laxalde, Perktold, Cimrman, Henriksen, Quintero, Harris,
  Archibald, Ribeiro, Pedregosa, {van Mulbregt}, and {SciPy 1.0
  Contributors}]{2020SciPy}
Pauli Virtanen, Ralf Gommers, Travis~E. Oliphant, Matt Haberland, Tyler Reddy,
  David Cournapeau, Evgeni Burovski, Pearu Peterson, Warren Weckesser, Jonathan
  Bright, St{\'e}fan~J. {van der Walt}, Matthew Brett, Joshua Wilson, K.~Jarrod
  Millman, Nikolay Mayorov, Andrew R.~J. Nelson, Eric Jones, Robert Kern, Eric
  Larson, C~J Carey, {\.I}lhan Polat, Yu~Feng, Eric~W. Moore, Jake
  {VanderPlas}, Denis Laxalde, Josef Perktold, Robert Cimrman, Ian Henriksen,
  E.~A. Quintero, Charles~R. Harris, Anne~M. Archibald, Ant{\^o}nio~H. Ribeiro,
  Fabian Pedregosa, Paul {van Mulbregt}, and {SciPy 1.0 Contributors}.
\newblock {{SciPy} 1.0: Fundamental Algorithms for Scientific Computing in
  Python}.
\newblock \emph{Nature Methods}, 17:\penalty0 261--272, 2020.

\bibitem[Watkins and Dayan(1992)]{watkins1992q}
Christopher~JCH Watkins and Peter Dayan.
\newblock Q-learning.
\newblock \emph{Machine learning}, 8\penalty0 (3-4):\penalty0 279--292, 1992.

\bibitem[Ye and Lim(2016)]{ye2016schubert}
Ke~Ye and Lek-Heng Lim.
\newblock Schubert varieties and distances between subspaces of different
  dimensions.
\newblock \emph{SIAM Journal on Matrix Analysis and Applications}, 37\penalty0
  (3):\penalty0 1176--1197, 2016.

\bibitem[Yu and Bertsekas(2009)]{yu2009basis}
Huizhen Yu and Dimitri~P Bertsekas.
\newblock Basis function adaptation methods for cost approximation in {MDP}.
\newblock In \emph{2009 IEEE Symposium on Adaptive Dynamic Programming and
  Reinforcement Learning}, pages 74--81. IEEE, 2009.

\end{thebibliography}

\clearpage

\onecolumn

\begin{appendix}

\hsize\textwidth
  \linewidth\hsize \toptitlebar {\centering
  {\Large\bfseries On The Effect of Auxiliary Tasks on Representation Dynamics: \\ Appendices \par}}
 \bottomtitlebar

\section{Additional results}

In this section, we state and prove some additional lemmas that are useful in proving the results stated in the main paper.

\begin{lemma}\label{lem:grassmann1}
    Let $x \in \mathbb{R}^d$, and let $(v_t)_{t \geq 0}$ be a sequence of vectors in $\mathbb{R}^d$ satisfying $v_t = f(t) x + o(f(t))$, for some function $f : [0, \infty) \rightarrow (0, \infty)$. Then $d(\langle v_t\rangle , \langle x\rangle ) \rightarrow 0$ as $t \rightarrow \infty$.
\end{lemma}
\begin{proof}
    The Grassmann distance $d(\langle v_t\rangle , \langle x\rangle )$ between two one-dimensional subspaces has a particular simple form, given by
    \begin{align*}
        d(\langle v_t \rangle , \langle x \rangle ) = \min\left( \arccos\left( \frac{\langle v_t, x \rangle}{\|v_t\| \|x\|} \right) , \arccos\left( \frac{\langle -v_t, x \rangle}{\|v_t\| \|x\|} \right) \right) \, .
    \end{align*}
    In our case, for sufficiently large $t$ this yields
    \begin{align*}
        d(\langle v_t\rangle , \langle x\rangle ) & = \arccos\left( \frac{\langle f(t) x + o(f(t)), x \rangle}{\| f(t) x + o(|f(t)|) \| \|x\|} \right) \\
        & = \arccos\left( \frac{\langle  x + o(1), x \rangle}{\| x + o(1) \| \|x\|} \right)\\
        & \rightarrow \arccos\left( \frac{\langle  x , x \rangle}{\| x \| \|x\|} \right)\\
        & = 0 \, .
    \end{align*}
\end{proof}

\begin{lemma}\label{lem:grassmann2}
    Let $U_1,\ldots,U_{|\mathcal{X}|}$ be a basis for $\mathbb{R}^{\statespace}$, let $K < |\mathcal{X}|$, and let $(a_{ij} |i \in [K], j \in [|\mathcal{X}|])$ be real coefficients.
    Let $0 < \beta_1 < \cdots < \beta_{|\mathcal{X}|}$
    , and consider time-dependent vectors $W_1(t),\ldots,W_d(t)$ defined by
    \begin{align*}
        W_i(t) = \sum_{j=1}^{|\mathcal{X}|} a_{ij} e^{-\beta_j t} U_j \, , \quad t \geq 0 \, .
    \end{align*}
    Then for almost all sets of coefficients $(a_{ij} |i \in [K], j \in [|\mathcal{X}|])$, we have
    \begin{align*}
        d(W_{1:K}(t) , U_{1:K}) \rightarrow 0 \, .
    \end{align*}
\end{lemma}

\begin{proof}
    Without loss of generality, we may take the vectors $U_1,\ldots,U_{|\mathcal{X}|}$ to be the canonical basis vectors. Under the assumptions of the theorem, we exclude initial conditions for which the matrix $A$ with $(i,j)$\textsuperscript{th} element $a_{ij}$ is not full rank. Note that under this condition, the matrix $A_t$ with $(k,i)$\textsuperscript{th} element $a_{ki}e^{\beta_i t}$ is also full rank for all but finitely many $t$. 
    By performing row reduction operations and scaling rows, for all such $t$ we may pass from $(W_{\repix}(t) \mid \repix \in [\repdim])$ to an alternative spanning set $(\widetilde{W}_{\repix}(t) \mid \repix \in [ \repdim])$ of the same subspace such that $\widetilde{W}_{\repix}(t) - U_\repix \in \langle U_{\repdim+1:|\mathcal{X}|}\rangle$, and $\|\widetilde{W}_\repix(t) - U_\repix\| = O(e^{-t(\beta_{\repdim+1} - \beta_\repix)}) = o(1)$. We therefore obtain an orthonormal basis for this subspace of the form $U_1 + o(1),\ldots, U_\repdim +o(1)$.
    
    We now use the singular value decomposition characterisation of Grassmann distance in Definition~\ref{def:grassmann-distance}. Since we have obtained an orthonormal basis for the subspace $\langle W_\repix(t) \mid \repix \in [\repdim] \rangle$, the top-$K$ singular values of the matrix $(\sum_{\repix=1}^\repdim U_\repix U_\repix^\top)(\sum_{\repix=1}^\repdim (U_\repix + o(1))( U_\repix + o(1))^\top )$ determine the Grassmann distance. However, this matrix is equal to $\text{diag}(1,\ldots,1,0,\ldots,0) + o(1)$, with $K$ entries of $1$ in the diagonal matrix.
    But the top-$K$ singular values this matrix are $1+o(1)$, and so the principal angles between the subspaces are $o(1)$, and hence the Grassmann distance between the subspaces is $o(1)$, as required.
\end{proof}

\begin{restatable}{lemma}{lemmaRewardMatrix}\label{lem:reward-matrix}
For $M \in \mathbb{N}$, let $(r^m)_{m=1}^M$ be independent random variables drawn from some fixed mean-zero distribution in $\mathscr{P}(\mathbb{R}^{\statespace\times\actionspace})$ such that the covariance between coordinates $(x, a), (y, a)$ is $\Sigma_{xy}$, independent of $a \in \mathcal{A}$. Let $(\mathbf{w}^m)_{m=1}^M$ be independent random variables taking values in $\mathbb{R}^{\repdim \times \actionspace}$, with columns drawn independently from $\mathcal{N}(0, (1/M)I)$.
Then $\sum_{m=1}^M r^m (\mathbf{w}^m)^\top$ converges (in distribution) to a mean-zero Gaussian distribution over $\mathbb{R}^{\statespace \times \repdim}$, with independent columns, and individual columns having covariance matrix $\Sigma$.
\end{restatable}

\begin{proof}
    The proof simply follows by noting that $\sum_{m=1}^M r^m \mathbf{w}^m$ may be written $1/\sqrt{M} \sum_{m=1}^M r^m \varepsilon^m$, with $(\varepsilon^m)_{m=1}^\infty$ i.i.d.~$N(0,I)$ random variables. The individual terms have the desired mean and variance, and the resulting converge in distribution now follows from the central limit theorem.
\end{proof}

\begin{restatable}{lemma}{lemmaWLimit}\label{lem:w-limit}
For fixed $M$, let $(\mathbf{w}^m)_{m=1}^M$, $\mathbf{w}^m \in \mathbb{R}^d$, be sampled i.i.d. according to $\mathcal{N}(0, \frac{1}{M}I)$. Then the following hold.
\begin{equation}
    \lim_{M \rightarrow \infty} \sum_{m=1}^M \mathbf{w}^m (\mathbf{w}^m)^\top = I \text{ and } \lim_{M \rightarrow \infty} \sum_{m=1}^M \mathbf{w}^m \overset{D}{=} \epsilon \sim \mathcal{N}(0, I)
\end{equation}
\end{restatable}
\begin{proof}
We prove two results on the limit of $W = \sum_{m=1}^M \mathbf{w}^m (\mathbf{w}^m)^\top$ as $k \rightarrow \infty$. First
\begin{align*}
    \lim_{M \rightarrow \infty} \sum_{m=1}^M \mathbf{w}^m (\mathbf{w}^m)^\top &\overset{P}{=} I \, , \\ 
    \intertext{which we observe by evaluating an arbitrary diagonal and off-diagonal element of $\sum_{m=1}^M \mathbf{w}^m (\mathbf{w}^m)^\top$. For the diagonal terms, note that}
    \left(\sum_{m=1}^M \mathbf{w}^m (\mathbf{w}^m)^\top\right) [j, j] &= \sum_{m=1}^M (\mathbf{w}^m_{j})^2
\end{align*}
Now observe that
\begin{align*}
    \mathbb{E}\left\lbrack \sum_{m=1}^M (\mathbf{w}^m_j)^2 \right\rbrack &= M \frac{1}{M} = 1  \, , \text{ and } \quad 
    \text{Var} \left(\sum_{m=1}^M (\mathbf{w}^m_j)^2\right) = M \frac{1}{M^2} \rightarrow 0 \\
\end{align*}
Similarly, for the off-diagonal terms, let $j \not= \ell$. Then we have
\begin{align*}
    \left( \sum_{m=1}^M \mathbf{w}^m (\mathbf{w}^m)^\top\right) [j, \ell] &= \sum_{m=1}^M \mathbf{w}^m_j \mathbf{w}^m_\ell \, ,
\end{align*}
and further
\begin{align*}
    \mathbb{E}\left\lbrack \sum_{m=1}^M \mathbf{w}^m_j \mathbf{w}^m_\ell \right\rbrack  = 0 \, , \text{ and } \quad
    \text{Var}\left(\sum_{m=1}^M \mathbf{w}^m_\ell \mathbf{w}^m_j\right) &= M \frac{1}{M^2} \rightarrow 0; \quad
\end{align*}
The limit in probability is immediately implied by Chebyshev's inequality. The result on $\sum_{m=1}^M \mathbf{w}^m$ follows immediately from part 1 and the fact that a sum of Gaussian random variables is another Gaussian random variable whose mean and variance in this case will be a standard normal.
\end{proof}
\section{Proofs}
\label{sec:proofs}
\lemODESoln*

\begin{proof}
    Equation~\eqref{eq:value-function-ode-solution} can be verified as a solution to Equation~\eqref{eq:value-function-ode} by direct differentiation. Uniqueness of the solution follows since this is an autonomous initial value problem that satisfies the Lipschitz condition, and so the Picard-Lindelh\"of theorem applies.
\end{proof}

\propOneValueFunction*

\begin{proof}
    By Assumption~\ref{assume:value-function-conditions}, $P^\pi$ is diagonalisable, with eigenbasis $U_1,\ldots,U_{|\mathcal{X}|}$, with corresponding eigenvalues $\lambda_{1:|\mathcal{X}|}$ with strictly decreasing magnitudes $|\lambda_1| > \cdots > |\lambda_{|\mathcal{X}|}|$. We note then that $\exp-(t(I - \gamma P^\pi))$ is also diagonaisable under the same basis, with eigenvalues $\exp(t (\gamma \lambda_i - 1))$, for $i=1,\ldots,|\mathcal{X}|$. We may therefore expand $V_0$ with respect to this eigenbasis, and write
    \begin{align*}
        V_0 - V^\pi = \sum_{i=1}^{|\mathcal{X}|} \alpha_i U_i \, ,
    \end{align*}
    for some $\alpha_{1:|\mathcal{X}|} \in \mathbb{R}^{|\mathcal{X}|}$. Now note from the differential equation \eqref{eq:value-function-ode-solution}, we have
    \begin{align*}
        V_t - V^\pi = \exp(-t(I - \gamma P^\pi)) (V_0 - V^\pi) = \sum_{i=1}^{|\mathcal{X}|} \alpha_i \exp(t(\gamma \lambda_i - 1)) U_i \, .
    \end{align*}
    Note that as $P^\pi$ is a stochastic matrix, we have $|\lambda_i| \leq 1$ for all $i=1,\ldots,|\mathcal{X}|$, and hence $\exp(t(\gamma \lambda_i - 1)) \rightarrow 0$ for all $i=1,\ldots,|\mathcal{X}|$. Further, $\exp(t(\gamma \lambda_i - 1)) = o(\exp(t(\gamma \lambda_1 - 1)))$ for all $i=2,\ldots,|\mathcal{X}|$. 
    We make the additional assumption that $\alpha_1 \not= 0$, which makes the `almost every initial condition' assumption in the statement precise. Under this assumption, we therefore have
    \begin{align*}
        V_t - V^\pi = \alpha_1 \exp(t(\gamma \lambda_1 - 1)) U_1 + \sum_{i=2}^{|\mathcal{X}|} \alpha_i \exp(t(\gamma \lambda_i - 1)) U_i = \alpha_1 \exp(t(\gamma \lambda_1 - 1)) U_1 + o(\exp(t(\gamma \lambda_1 - 1))) \, .
    \end{align*}
    Then Lemma~\ref{lem:grassmann1} applies to give $d(\langle V_t - V^\pi \rangle, \langle U_1 \rangle) \rightarrow 0$, as required.
\end{proof}

\propManyValueFunctions*

\begin{proof}
    Expanding $V^{(\repix)}_0 - V^\pi$ with respect to $U_1,\ldots,U_{|\mathcal{X}|}$ for each $\repix=1,\ldots,|\mathcal{X}|$, we obtain expressions of the form
    \begin{align*}
        V^{(\repix)}_0 - V^\pi = \sum_{i=1}^{|\mathcal{X}|} a_{\repix i} U_i \, .
    \end{align*}
    By the ODE solution in Lemma~\ref{lem:ode-soln}, we then have
    \begin{align*}
        V^{(\repix)}_t - V^\pi = \sum_{i=1}^{|\mathcal{X}|} a_{\repix i}  e^{-t(1-\gamma\lambda_i)} U_i \, .
    \end{align*}
    We may now apply Lemma~\ref{lem:grassmann2} to obtain the desired result.
\end{proof}

\lemCoupledDynamics*

\begin{proof}
    This follows immediately by computing the derivatives in Equations~\eqref{eq:phi-ode} \& \eqref{eq:w-ode}, and so we omit the direct calculations.
\end{proof}

\thmInfiniteHeads*

\begin{proof}
We write the dynamics on $\Phi^M_t$ as follows and apply the results of Lemma~\ref{lem:w-limit}. We first consider the scaled initialization setting (implicitly setting the learning rate $\alpha=1$), where we find
\begin{align}
    \partial_t \Phi^M_t &= (I - \gamma P^\pi)\Phi_t^M \sum_{m=1}^M \mathbf{w}^m (\mathbf{w}^m)^\top + \sum_{m=1}^M R^{\pi} (\mathbf{w}^m)^\top \\
    \lim_{M \rightarrow \infty} \partial_t \Phi^M_t &= (I - \gamma P^\pi)\Phi_t^M \lim_{M \rightarrow \infty}\sum_{m=1}^M \mathbf{w}^m (\mathbf{w}^m)^\top  + \lim_{M \rightarrow \infty} R^\pi(\sum_{m=1}^M \mathbf{w}^m)^\top \\
    &\overset{D}{=} (I - \gamma P^\pi )\Phi_t^M I + R^\pi \epsilon^\top, \; \epsilon \sim \mathcal{N}(0, I).
\end{align}

We further observe that, for any finite interval, in the setting of zero reward we obtain \textit{uniform} convergence of the induced trajectory $\Phi_t^M$ to the trajectory of the limiting dynamics. We first observe that for a fixed initialization, we have that the induced dynamics are linear (in the zero-reward setting, affine otherwise) function of $\Phi^M_t$, and so 
\begin{align*}
    \partial_t \Phi^M_t &= (I - \gamma P^\pi) \Phi^M_t \sum_{m=1}^M w^m (w^m)^\top = \mathcal{L}^M \Phi^M_t  \\
    &\text{ where $\mathcal{L}^M(A) = (I - \gamma P^\pi) A \sum_{m=1}^M w^m (w^m)^\top$} \\
    \implies \Phi^M_t &= \exp(t \mathcal{L}^M)\Phi^M_0\; .
\intertext{Because the function $t \mapsto \exp (t A)$ is Lipschitz on a bounded interval for any $A$, this implies that for any finite interval $[0, T]$, the functions $t \mapsto \Phi^M_t$, as well as limiting solution, are $L$-Lipschitz for some $L$. Further, since the exponential is continuous, }
\lim_{M \rightarrow \infty} \Phi^M_t &= \lim_{M \rightarrow \infty} \exp(t \mathcal{L}^M) \Phi^M_0 = \exp(t \lim_{M \rightarrow \infty} \mathcal{L}^M) \Phi_0 \\
&= \exp (-t(I - \gamma P^\pi)) \Phi_0 = \Phi^\infty_t \;.
\intertext{Therefore, the functions $t \mapsto \Phi^M_t$ are $L$-Lipschitz and converge to the limit $\Phi^\infty_t$ on the interval $[0, T]$, which implies that they converge uniformly. }
\end{align*}

To evaluate the scaled learning rate setting, we observe that we now have
\begin{align}
    \partial_t \Phi^M_t &= \frac{1}{M} (I - \gamma P^\pi)\Phi_t^M \sum_{m=1}^M \mathbf{w}^m (\mathbf{w}^m)^\top + \sum_{m=1}^M R^{\pi} (\mathbf{w}^m)^\top \\
    \lim_{M \rightarrow \infty} \partial_t \Phi^M_t &= (I - \gamma P^\pi)\Phi_t^M \lim_{M \rightarrow \infty} \frac{1}{M}\sum_{m=1}^M \mathbf{w}^m (\mathbf{w}^m)^\top  + \lim_{M \rightarrow \infty} \frac{1}{M}R^\pi(\sum_{m=1}^M \mathbf{w}^m)^\top \\
    & = (I - \gamma P^\pi )\Phi_t^M I . \\
       \implies \lim_{M \rightarrow \infty} \Phi^M_t & = \exp(-t(I - \gamma P^\pi)) \Phi_0\, ,
\end{align}
almost surely. 
The principal difference between this and the scaled initialization setting is that here we divide the $R^\pi \mathbf{w}^\top$ term by $\frac{1}{M}$, whereas the scaled initialization is equivalent to scaling by $\frac{1}{\sqrt{M}}$. Therefore the scaled learning rate limit can be computed by the law of large numbers and converges in probability to its mean (zero), whereas under the scaled initialization it converges via the central limit theorem to a Gaussian distribution. 
\end{proof}

\propSubspaceConvergence*

\begin{proof}
    As described in the proof of Theorem~\ref{thm:infinite-heads}, we have $\Phi_t = \exp(-t(I - \gamma P^\pi))(\Phi_0 - \Phi_\infty) + \Phi_\infty$. Under Assumption~\ref{assume:value-function-conditions}, we may now apply an analogous argument as in Proposition~\ref{prop:many-value-functions} to the columns of $\Phi_t - \Phi_\infty$, and apply Lemma~\ref{lem:grassmann2} to obtain the desired result.
\end{proof}

\ThmDistribution*

\begin{proof}
We recall from Theorem~\ref{thm:infinite-heads} that the limiting dynamics follow the distribution
\begin{align}
    \lim_{t \rightarrow \infty} \lim_{M \rightarrow \infty} \Phi_t &\overset{D}{=} \lim_{t \rightarrow \infty} \exp(-t(I - \gamma P^\pi)) (\Phi_0 - (I -\gamma P^\pi)^{-1} Z_\Sigma) + (I - \gamma P^\pi)^{-1} Z_\Sigma\\
    & \overset{D}{=} (I - \gamma P^\pi) ^{-1} Z_\Sigma
\end{align}
for which we can straightforwardly apply known properties of Gaussian distributions: namely, that the distribution of a linear transformation $A$ of a Gaussian random variable with parameters $\mu, \Sigma$ is also Gaussian with mean $A\mu$ and covariance $A \Sigma A^\top$. Letting $A=(I - \gamma P^\pi)$ therefore gives the desired result.
\end{proof}

\propSubspaceConvergenceRC*

\begin{proof}
    As described in the proof of Theorem~\ref{thm:distribution}, we have $\Phi_t = \exp(-t(I - \gamma P^\pi))(\Phi_0 - (I - \gamma P^\pi)^{-1} Z_\Sigma ) + (I - \gamma P^\pi)^{-1} Z_\Sigma $. Under Assumption~\ref{assume:value-function-conditions}, we may now apply an analogous argument as in Proposition~\ref{prop:many-value-functions} to the columns of $\Phi_t - (I - \gamma P^\pi)^{-1} Z_\Sigma$, and apply Lemma~\ref{lem:grassmann2} to obtain the desired result.
\end{proof}

\section{Additional results from Table~\ref{tab:theory}}

We begin this section by noting the following property of systems following linear dynamics.
\begin{lemma}\label{lem:dynamics-aux}
Let $\Phi_t \in \mathbb{R}^{\statespace \times M}$ follow the dynamics $\partial_t \Phi_t \overset{D}{=} A \Phi_t + B$, where $A$ is a linear operator for which all eigenvalues have negative real part, and $B$ is a vector. Then 
\begin{align}
    \lim_{t \rightarrow \infty} \Phi_t = -A^{-1}B \, .
\end{align}
Further, if $A$ is diagonalisable, with all eigenvalues of different magnitudes, 
\begin{align}
    \lim_{t \rightarrow \infty} d( \langle \Phi_t - \Phi_\infty \rangle, \langle U_{1:\repdim}(A) \rangle) = 0 \, ,
\end{align}
where $U_i(A)$ is the eigenvector of $A$ corresponding to the eigenvalue with $i$\textsuperscript{th} largest magnitude.
\end{lemma}
\begin{proof}
    We observe that the dynamics $\partial_t \Phi_t = A \Phi_t$ induce the trajectory 
    \begin{equation}
       \Phi_t = \exp(tA)\Phi_0 + (I - \exp(tA))(-A^{-1}B)\; ,
    \end{equation}
    with limit $\Phi_\infty = - A^{-1}B$. When $A$ is diagonalizable, we can therefore straightforwardly apply the results of Lemma~\ref{lem:grassmann2} to get that the limiting subspace will be characterized by the top $k$ eigenvectors of $A$. In the settings we are interested in, $A = - ( I - \gamma P^\pi)$ for some $\pi$ and some $\gamma$, and so the principal eigenvectors of $A$ will be the principal eigenvectors of $P^\pi$.
\end{proof}

The following two theorems characterize the learning dynamics under the past policies and multiple timescale auxiliary tasks listed in Table~\ref{tab:theory}. With these characterizations, it becomes straightforward to deduce $\Phi_\infty$ and the limiting subspace error as a direct consequence of the previous lemma.

\begin{theorem} \label{thm:pastpolicies-aux}
Let $\pi_1, \dots, \pi_L$ be a fixed set of policies. Given fixed $M$ and $L$, we define the indexing function $i_m = \lceil\frac{L}{m} \rceil$ for $m \in [1, M]$. Let $\Phi^M_t$ follow the dynamics
\begin{align}
    \partial_t \Phi_t^M &= \sum_{m=1}^M -( (I - \gamma P^{\pi_{i_m}})\Phi_t^M \mathbf{w}^m_t + R^{\pi_{i_m}})(\mathbf{w}^m_t)^\top
\end{align}
Then $\Phi_t^M$ satisfies the following dynamics and trajectory in the limit as $M \rightarrow \infty$, where $\bar{\pi} = \sum_{i=1}^L \pi_i$ and $\epsilon_i \in \mathbb{R}^d$ is an isotropic Gaussian with variance $\frac{1}{L}$. Note that we cannot naively average the rewards without changing the variance of the induced distribution unless $R^{\pi_i} = R^{\pi_j}$ for all $i,j$.
\begin{align}
    \lim_{M \rightarrow \infty} \partial_t \Phi_t^M &\overset{D}{=} -(I - \gamma P^{\bar{\pi}})\Phi_t + \sum_{i=1}^L R^{\pi_{i}} \epsilon_{i}\\
    \lim_{M \rightarrow \infty} \Phi_t^M &\overset{D}{=} \exp (-t(I - \gamma P^{\bar{\pi}} )) (\Phi_0 - \Phi_\infty ) + (I - \gamma P^{\bar{\pi}} )^{-1} \bigg ( \sum_{i=1}^L R^{\pi_{i}} \epsilon_{i}^\top \bigg )
\end{align}

\end{theorem}
\begin{proof}
The result on the trajectories follows immediately from the result on the dynamics, so it suffices to prove convergence of the dynamics. We approach this problem by decomposing the dynamics of $\Phi^M_t$ as follows.
\begin{equation}
    \partial_t \Phi_t^M = \sum_{m=1}^M - (I - \gamma P^{\pi_{i_m}})\Phi_t^M \mathbf{w}^m_t (\mathbf{w}^m_t)^\top - \sum_{m=1}^M R^{\pi_{i_m}}(\mathbf{w}^m_t)^\top \, .
\end{equation}

We first consider the random variables in the term which includes the rewards $R^\pi$. For this, we can directly apply the results from the previous theorems to the random variables $\epsilon_j = \sum_{m : i_m = j} w^m$, whose limiting variance is easily computed to be 
\begin{equation}
    \lim_{M \rightarrow \infty} \text{Var}\left(\sum_{m : i_m = j} \mathbf{w}^m\right) = \lim_{M \rightarrow \infty} \sum_{\lfloor \frac{j}{n}M \rfloor}^{\lfloor \frac{j+1}{n}M \rfloor} \frac{1}{M}I = \frac{1}{L}I \, . 
\end{equation}
For the term which depends on $\Phi_t$, we see
\begin{align}
    \sum_{m=1}^M - (I - \gamma P^{\pi_{i_m}} )\Phi^M_t \mathbf{w}^m_t (\mathbf{w}^m_t)^\top &= \sum_{i=1}^L \sum_{m : i_m = i}^M - (I - \gamma P^{\pi_{i_m}} )\Phi^M_t \mathbf{w}^m_t (\mathbf{w}^m_t)^\top \\
    &= \sum_{i=1}^L  - (I - \gamma P^{\pi_{i}} )\Phi^M_t \sum_{m : i_m = i}^M \mathbf{w}^m_t (\mathbf{w}^m_t)^\top  \, . \\
    \intertext{Since $L$ is finite and fixed, $\sum_{m:i_m=i}^M\mathbf{w}^m_t(\mathbf{w}^m_t)^\top$ converges to $\frac{1}{L}I$}
    & \underset{M \rightarrow \infty}{\longrightarrow} \sum_{i=1}^L - (I - \gamma P^{\pi_{i}} )\Phi^M_t \frac{1}{L}I \\
    &= -( I - \gamma \frac{1}{L}\sum_{i=1}^L P^{\pi_i})\Phi_t^M \\
    &= -(I - \gamma P^{\bar{\pi}} )\Phi_t^M \, .
\end{align}
And so the limiting distribution becomes
\begin{equation}
    \lim_{M \rightarrow \infty } \partial_t \Phi_t^M = -(I - \gamma P^{\bar{\pi}} )\Phi_t^M - \bigg ( \sum_{i=1}^L R^{\pi_{i}} \epsilon_{i} \bigg ) \, .
\end{equation}
\end{proof}
\begin{corollary}
The above result can be readily adapted to the setting in which each head predicts a randomly selected (deterministic) policy in MDPs with finite state and action spaces. Let $L = |\actionspace| ^{|\statespace|}$, $\{\pi_1, \dots, \pi_L\}$ be an enumeration of $\actionspace ^\statespace$, and $i_m$ denote the index of the policy randomly assigned to head $m$; then the above result still holds, and $\bar{\pi}$ is the uniform policy.
\end{corollary}
\begin{theorem} \label{thm:multiple-timescales}
We consider the task of predicting the value functions of a fixed policy under multiple discount rates $\gamma_1, \dots, \gamma_L$. For fixed $M$, $L$, let $i_m$ denote the indexing function defined in Theorem \ref{thm:pastpolicies-aux} Let $\Phi_t$ follow the dynamics
\begin{align}
    \partial_t \Phi_t^M &= \sum_{m=1}^M -( (I - \gamma_{i_m} P^{\pi})\Phi_t^M \mathbf{w}^m_t + R^{\pi})(\mathbf{w}^m_t)^\top \, .
\end{align}
Then the limiting dynamics as $M\rightarrow \infty$ of $\Phi_t^M$ are as follows, where $\bar{\gamma} = \sum \frac{1}{L} \gamma_i$
\begin{align}
    \lim_{M \rightarrow \infty} \partial_t \Phi_t^M &\overset{D}{=} -(I - \bar{\gamma} P^{\pi})\Phi_t + R^{\pi} \epsilon^\top\\
    \intertext{and}
    \lim_{M \rightarrow \infty} \Phi_t^M &\overset{D}{=} \exp (-t(I - \gamma P^{\bar{\pi}} )) (\Phi_0 - \Phi_\infty ) + (I - \gamma P^{\pi} )^{-1} R^{\pi} \epsilon^\top)\; .
\end{align}
\end{theorem}

\begin{proof}
We follow a similar derivation as for Theorem~\ref{thm:pastpolicies-aux} in deriving the component of the dynamics which depends on $\Phi^M_t$. The result of Theorem~\ref{thm:infinite-heads} immediately applies to the $\sum R^\pi (\mathbf{w}^m_t)^\top$ term:
\begin{align}
    \sum_{m=1}^M - (I - \gamma_{i_m} P^\pi )\Phi^M_t \mathbf{w}^m_t (\mathbf{w}^m_t)^\top &= \sum_{i=1}^L \sum_{m : i_m = i}^M - (I - \gamma_i P^\pi )\Phi^M_t \mathbf{w}^m_t (\mathbf{w}^m_t)^\top \\
    &= \sum_{i=1}^L  - (I - \gamma_i P^{\pi} )\Phi^M_t \sum_{m : i_m = i}^M \mathbf{w}^m_t (\mathbf{w}^m_t)^\top \, . \\
    \intertext{Since $L$ is finite and fixed, $\sum_{m:i_m=i}^M\mathbf{w}^m_t(\mathbf{w}^m_t)^\top$ converges to $\frac{1}{L}I$ as before:}
    & \underset{M \rightarrow \infty}{\longrightarrow} \sum_{i=1}^L - (I - \gamma_i P^{\pi} )\Phi^M_t \frac{1}{L}I \\
    &= -( I -\sum_{i=1}^L \frac{\gamma_i}{L}  P^{\pi})\Phi_t^M \\
    &= -(I - \bar{\gamma} P^{{\pi}} )\Phi_t^M \, .
\end{align}
\end{proof}

\section{Experimental details}
\label{sec:experiment-details}

\subsection{Experimental details for Figure~\ref{fig:feature-viz}}

In our evaluations of the evolution of single feature vectors, we compute the continuous-time feature evolution defined in Equation~\eqref{eq:phi-ode}, using $P^\pi$ defined by a random walk on a simple Four-Rooms Gridworld with no reward. We use a randomly initialized representation $\Phi \in \mathbb{R}^{ |\statespace| \times 10}$, and use a single column of this matrix in our feature visualization (we observed similar behaviour in each feature). To compute trajectories, we use the SciPy ODE solver \texttt{solve\_ivp} \citep{2020SciPy}.

\subsection{Experimental details for Section~\ref{sec:feature-generalisation}}

Here, we provide details of the environment used in producing Figure~\ref{sec:feature-generalisation}. The environment is a 30-state chain, with two actions, \texttt{left} and \texttt{right}, which move the agent one state to the left or right, respectively. When the agent cannot move further left or right (due to being at an end state of the chain), the result of the corresponding action keeps the agent in the same state. There is additionally environment stochasticity of $0.01$, meaning that with this probability, a uniformly random action is executed instead. This stochasticity ensures that $P^\pi$ satisfies the conditions of Assumption~\ref{assume:value-function-conditions}. Taking the action \text{left} in the left-most state incurs a reward of $+2$, and taking the action \texttt{right} in the right-most state incurs a reward of $+1$; all other rewards are zero.

\subsection{Experimental details for Section~\ref{sec:deep-rl}}

We modify a base Double DQN agent \citep{van2016deep} and evaluate on the ALE without sticky actions \citep{bellemare2013arcade}. Our agents are implemented in Jax \citep{jax2018github}, and are based on the DQN Zoo \citep{dqnzoo2020github}. Unless otherwise mentioned, all hyperparameters are as for the default Double DQN agent, with the exception of the epsilon parameter in the evaluation policy, which is set to 0.001 in all agents, and the optimizer, which for agents using auxiliary tasks CV, REM and Ensemble is Adam with epsilon $0.1/32^2$, and a lightly tuned learning rate; see below for further details.

Experimental results shown in bar plots, such as Figures~\ref{fig:naux}  and~\ref{fig:rc_sweep}, report a ``relative score'' which is the per-game score normalized by the maximum average score achieved by any agent or configuration. The same, per-game, normalization values are used for all such figures.

\textbf{Auxiliary task details.} In this section, we describe the implementations of all auxiliary tasks considered in the main paper.
\begin{itemize}
    \item \emph{QR-DQN.} The implementation and hyperparameters match QR-DQN-1 in \citet{dabney2018distributional}.
    \item \emph{DDQN+RC.} We use a many-head DQN network which is identical to the standard neural network used for DQN, except that the output dimension is $(M+1) \times |\actionspace|$ instead of $|\actionspace|$, where $M$ is the number of auxiliary heads. Random cumulants are generated using a separate neural network with the same architecture as a standard DQN, but with output dimension equal to the number of auxiliary heads. The width of the Huber loss for each auxiliary head is equal to the number of auxiliary tasks. Let $\phi(x) \in \mathbb{R}^M$ be the output of the cumulant network given input observation $x$, with $M$ the number of auxiliary heads. Then the cumulant for auxiliary head $m$, at time step $t$, is given by $c_t = s \times (\phi(x_{t+1}) - \phi(x_t))$, where $s \in \mathbb{R}$ is a scaling factor. We performed a small hyperparameter sweep over scaling factors in $\{1, 10, 100, 500\}$, finding $s = 100$ to provide the best performance and use this value for all reported experiments. Note that this auxiliary task and the details are nearly identical to the \emph{CumulantValues} auxiliary task of \citet{dabney2020value}, except that we do not pass the values through a tanh non-linearity as this did not appear to have any impact in practice. We performed a hyperparameter sweep over learning rates and gradient norm clipping for this agent, considering learning rates $\{0.00025, 0.0001, 0.00005\}$ and gradient clipping in $\{10, 40\}$. We found that a learning rate of $0.00005$ and gradient norm clipping of $40$ to work best and use these values for all experiments.
    \item \emph{DDQN+REM.} We use a many-head variant of Double DQN, with heads trained according to the REM loss of \citet{agarwal2019striving}. For the agent's policy, an argmax over a uniform average of the heads is used. We swept over learning rates of $0.0001$ and $0.00005$, generally finding $0.00005$ to perform best. 
    \item \emph{DDQN+Ensemble.} As for the REM auxiliary task, we use a many-head variant of Double DQN. Each head is trained using its own double DQN loss, and the resulting losses are averaged. For the agent's policy, an argmax over a uniform average of the heads is used. We swept over learning rates of $0.0001$ and $0.00005$, generally finding $0.00005$ to perform best.
\end{itemize}

\textbf{Modified dense-reward games.} We modified four Atari games (Pong, MsPacman, Seaquest, and Q*bert) to obtain sparse, harder versions of these games to test the performance of random cumulants and other auxiliary tasks. The details of these games are given below. In each case a low-valued, commonly encountered reward is `censored', which means that during training the agent observes a reward of $0$ instead of the targeted reward. When evaluated, and thus for all empirical results reported, the standard uncensored rewards are reported.
\begin{itemize}
    \item \emph{Sparse Pong.} All negative rewards are censored (i.e. set to 0 before being fed to the agent), so the agent receives a reward of +1 for scoring against the opponent, but no reward when it concedes a point to the opponent. As $0$, $1$, and $-1$ are the only rewards in Pong, this modification makes Pong significantly harder. The agent can no longer learn to `avoid losing points`, but can only improve by learning to score points directly.
    \item \emph{Sparse MsPacman.} All rewards less than or equal to 10 are censored. This corresponds to rewards for the numerous small pellets that MsPacman eats, but not the larger pellets or ghosts. Each level ends when all of the small pellets are consumed, thus, by hiding these from the agent we may have significantly changed the primary incentive for the agent to advance the game.
    \item \emph{Sparse Seaquest.} All rewards less than or equal to 20 are censored. This corresponds to the rewards for shooting the sharks underwater, but not the rewards for picking up divers or surfacing. Additionally, even the rewards for sharks increase beyond this level, and thus become visible, once the agent has surfaced and collected enough divers.
    \item \emph{Sparse Q*bert.} All rewards less than or equal to 25 are censored. These are the rewards for flipping the colour of a tile, which is the primary source of reward and the mechanism for advancing to the next level of the game. Once all tiles are flipped, the agent will go to the next level. However, the agent can still observe rewards for going to the next level and for dispatching the enemies.
\end{itemize}

As described in the main paper,we found that the sparse versions of MsPacman, Seaquest, and Q*bert were too difficult for any agent we tested to achieve a reasonable level of performance. In Figure~\ref{fig:sparse-learning-curves}, we display the performance of several auxiliary tasks on these games, noting that the performance achieved is extremely low in comparison to the agents trained on the standard versions of these games (see Figure~\ref{fig:deep-rl}).

\begin{figure}[!htb]
    \centering
    \null
    \hfill
    \includegraphics[keepaspectratio,width=\textwidth]{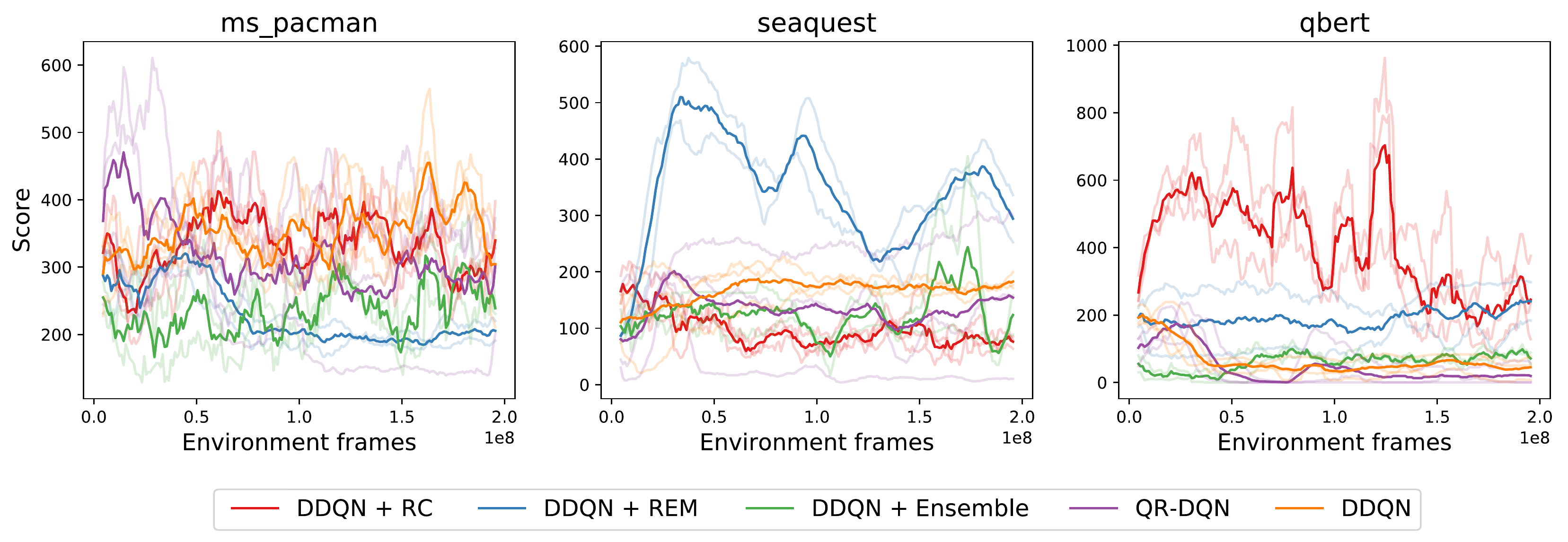}
    \caption{Learning curves on sparsified MsPacman (left), sparsified Seaquest (centre), and sparisifed Q*bert (right).}
    \label{fig:sparse-learning-curves}
\end{figure}

\textbf{Hyperparameter sweeps.} In Figure~\ref{fig:rc_sweep} we vary the weight of the auxiliary loss for the random cumulants agents, with the aim of understanding how this hyperparameters affect each method's performance. Next, in Figures~\ref{fig:ens_sweep} and~\ref{fig:rem_sweep} we present the results of a hyperparameter sweep for Ensemble and REM respectively. For these two, since there is no separate auxiliary loss as in RC, we vary number of heads and the learning rate. Results presented in the main text use the best settings for each algorithm found from these sweeps.

\begin{figure}[!htb]
    \centering
    \includegraphics[keepaspectratio,width=.75\textwidth]{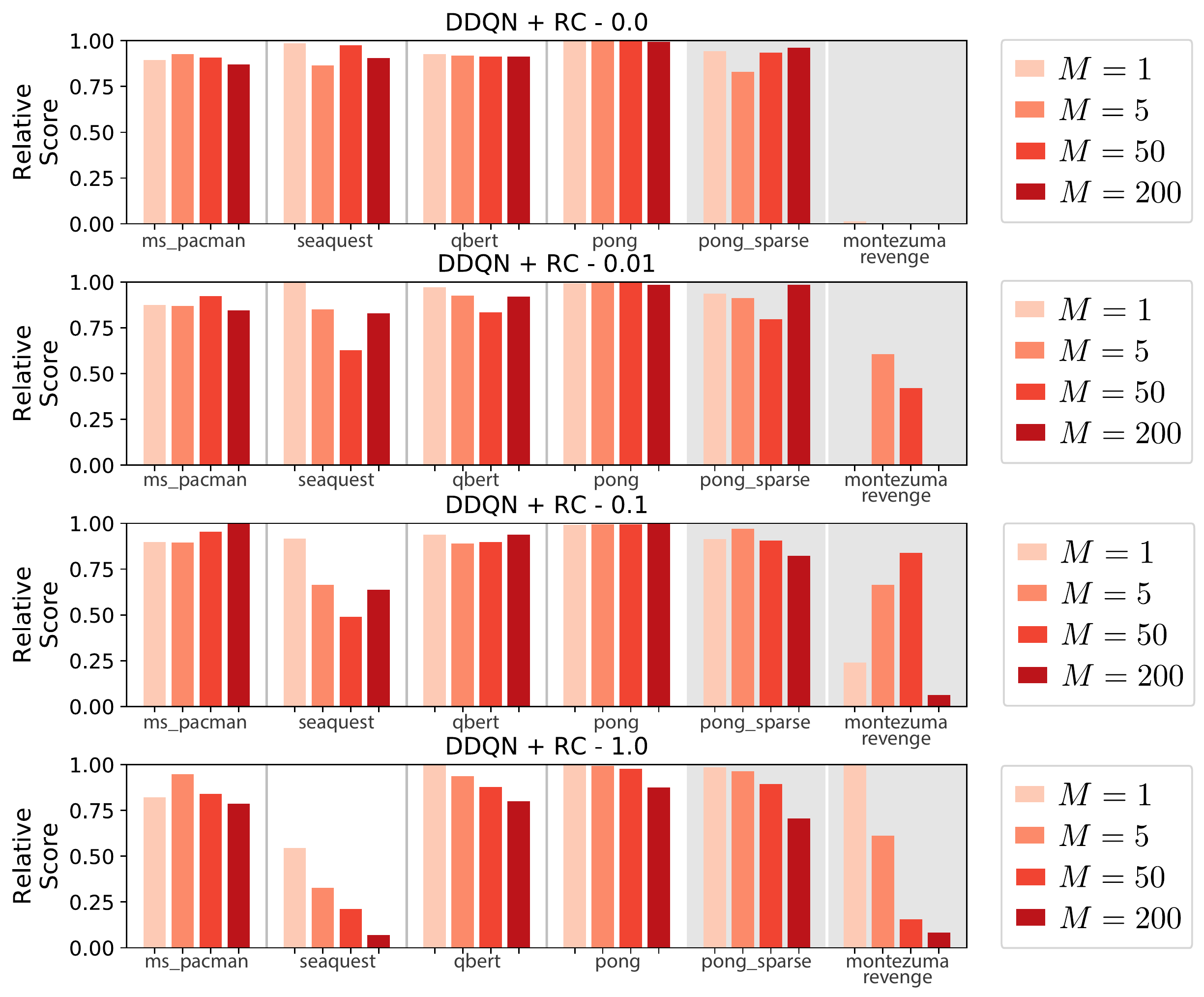}
    \caption{Results of hyper-parameter sweep for Random Cumulant (RC) method, where each row is for a different value of multiplicative scale applied to the auxiliary losses and each bar corresponds to the number of auxiliary heads ($M$). Note that the first row of results corresponds to initializing a network with the auxiliary heads, but setting the weight to zero, effectively disabling the auxiliary task.}
    \label{fig:rc_sweep}
\end{figure}

\begin{figure}[!htb]
    \centering
    \includegraphics[keepaspectratio,width=.75\textwidth]{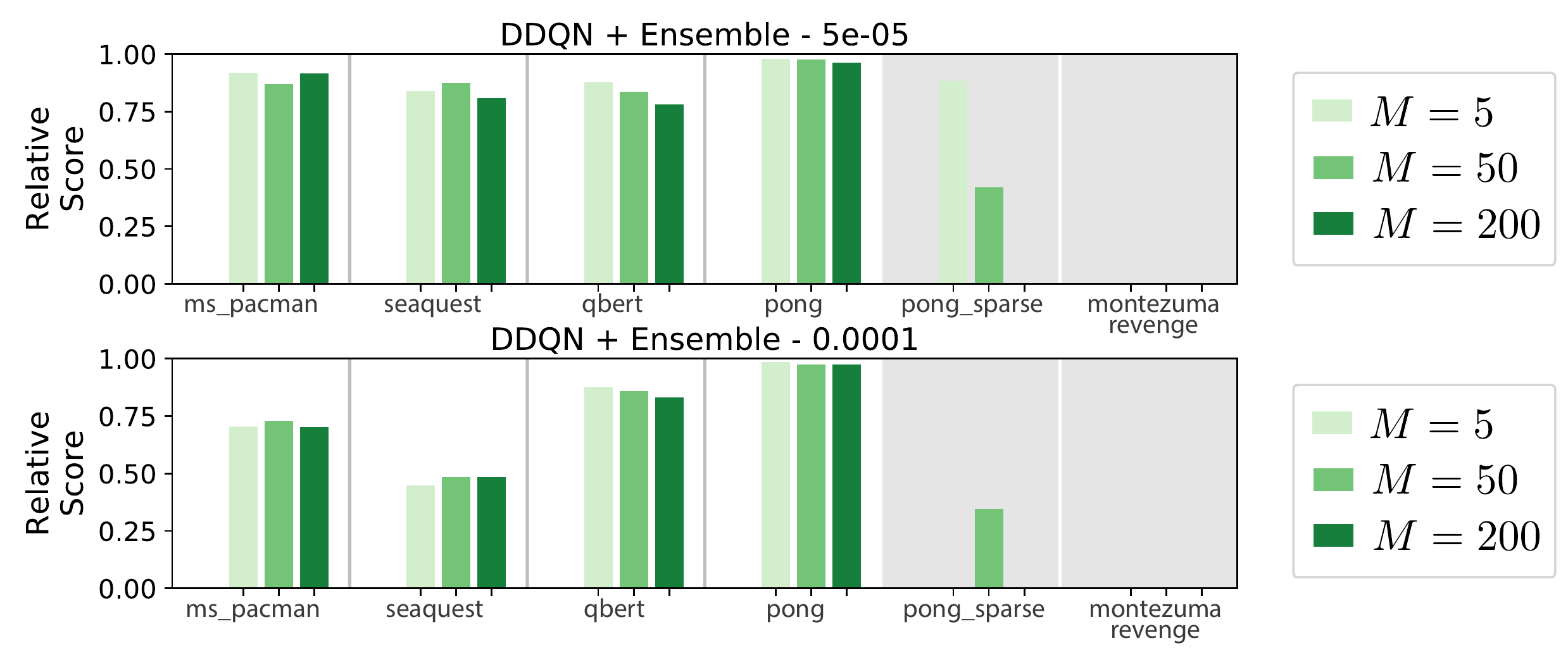}
    \caption{Results of hyper-parameter sweep for the Ensemble method, where each row is for a different learning rate and each bar corresponds to the number of auxiliary heads ($M$).}
    \label{fig:ens_sweep}
\end{figure}

\begin{figure}[!htb]
    \centering
    \includegraphics[keepaspectratio,width=.75\textwidth]{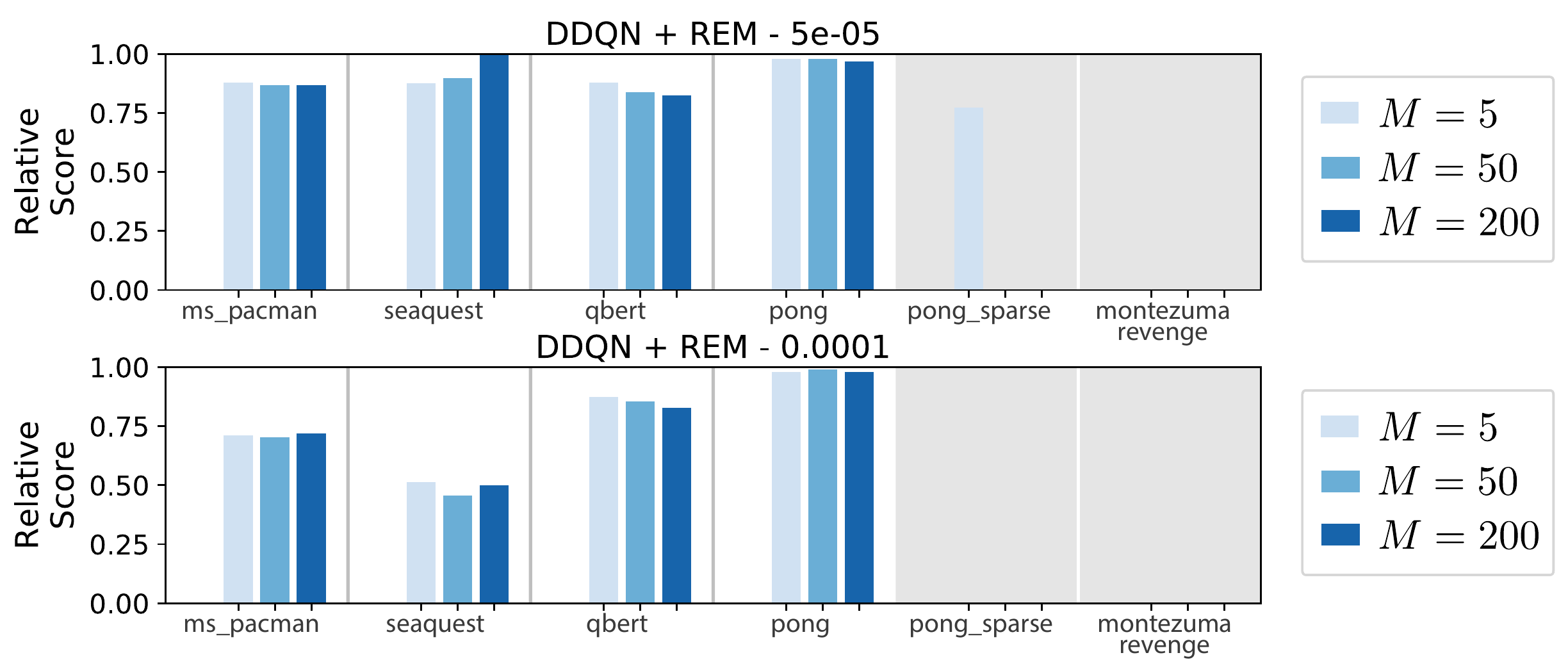}
    \caption{Results of hyper-parameter sweep for the REM method, where each row is for a different learning rate and each bar corresponds to the number of auxiliary heads ($M$).}
    \label{fig:rem_sweep}
\end{figure}

\clearpage

\section{Extensions beyond one-step temporal difference learning}\label{sec:beyond-one-step}

Our analysis in the main paper has focused on the case of learning dynamics under one-step temporal difference learning. This choice is largely because one-step temporal difference learning is such a popular algorithm, not because the results do not hold more generally. In this section, we describe the elements of analogous results for $n$-step learning and TD($\lambda$) for interested readers. We focus on the case of value function dynamics, and believe extensions of the representation dynamics analysis in the main paper along these lines will be interesting directions for future work.

\subsection{Temporal difference learning with $n$-step returns}

In the case of $n$-step returns, the dynamics on the value function $(V_t)_{t \geq 0}$ are given by
\begin{align*}
    \partial_t V_t(x) = \mathbb{E}_\pi\left\lbrack \sum_{k=0}^{n-1} \gamma^k R_k + \gamma^n V_t(X_n)\middle| X_0 = x \right\rbrack - V_t(x) \, .
\end{align*}
In full vector notation, we have
\begin{align*}
    \partial_t V_t = -(I - \gamma^n (P^\pi)^n) V_t + \left\lbrack \sum_{k=0}^{n-1} (\gamma P^\pi)^k \right\rbrack R^\pi \, .
\end{align*}
The solution to this differential equation is
\begin{align*}
    V_t = \exp( -t (I - (\gamma P^\pi)^n ) )(V_0 - V^\pi) + V^\pi \, .
\end{align*}
This bears a close relationship with the result obtained for $1$-step temporal difference learning in the main paper. As expected, we obtain the same limit point. Further, under Assumption~\ref{assume:value-function-conditions}, $(P^\pi)^n$ has the same eigenvectors as $P^\pi$, and so results analogous to Propositions~\ref{prop:one-value-function} \& \ref{prop:many-value-functions} hold for $n$-step temporal difference learning too under these conditions.

\subsection{Temporal difference learning with $\lambda$-returns}

In the case of temporal difference learning with $\lambda$-returns (for $\lambda \in [0,1)$), the dynamics on the value function $(V_t)_{t \geq 0}$ are given by
\begin{align*}
    \partial_t V_t(x) = \mathbb{E}_\pi\left\lbrack \sum_{k=0}^{\infty} (\lambda\gamma)^k (P^\pi)^k ( R^\pi + \gamma P^\pi V_t(X_{k+1}) - V_t(X_k))\middle| X_0 = x \right\rbrack - V_t(x) \, .
\end{align*}
In full vector notation, we have
\begin{align*}
    \partial_t V_t = \sum_{k=0}^{\infty} (\lambda\gamma)^k (P^\pi)^k ( R^\pi + \gamma P^\pi V_t - V_t)
\end{align*}
The solution to this differential equation is
\begin{align*}
    V_t = \exp\left(t\left((1-\lambda) \sum_{k=1}^\infty \lambda^{k-1}\gamma^k (P^\pi)^k - I\right)\right) (V_0 - V^\pi) + V^\pi \, .
\end{align*}
As with $n$-step temporal difference learning, this bears a close relationship with the result obtained for $1$-step temporal difference learning in the main paper. As expected, we obtain the same limit point. Further, under Assumption~\ref{assume:value-function-conditions}, each $(P^\pi)^k$ has the same eigenvectors as $P^\pi$, and so results analogous to Propositions~\ref{prop:one-value-function} \& \ref{prop:many-value-functions} hold for $n$-step temporal difference learning too under these conditions.

\section{Beyond diagonalisability assumptions}\label{sec:more-general-value-function-results}

In this section, we briefly describe extensions of the results of the main paper in scenarios where Assumption~\ref{assume:value-function-conditions} does not hold. There are two main cases we consider: (i) those in which $P^\pi$ is still diagonalisable, but does not have all eigenvalues with distinct magnitudes; and (ii) those in which $P^\pi$ is not diagonalisable.

In the former case, we do not have the different convergence rates of coefficients of different eigenvectors as in the proof of Proposition~\ref{prop:many-value-functions}. By similar arguments we can still deduce convergence of $V_t$ to the span of the eigenspaces with highest magnitude eigenvalues, but we can no longer deduce convergence to individual eigenspaces if there are several other eigenvalues with the same magnitude as the eigenvalue concerned. Note also that this includes the case where the matrix $P^\pi$ is complex- but not real-diagonalisable, since in such case non-real eigenvalues must come in conjugate pairs (which are necessarily of the same absolute value).

In the latter case, we no longer have an eigenbasis for $\mathbb{R}^{\statespace}$ based on $P^\pi$. However, we can consider the Jordan normal decomposition, and may still recover analogous results to those in main paper, where convergence is now to the subspaces generated by \emph{Jordan blocks} with high absolute value eigenvalues. See \citet{parr2008analysis} for further commentary on Jordan normal decompositions in feature analysis.

\section{Further discussion of features and operator decompositions}\label{sec:feature-selection}

Proto-value functions (PVFs), were first defined by \citet{mahadevan2007proto} as the eigenvectors of the \textit{incidence matrix} induced by the environment transition matrix $P$. In the ensuing years, the term PVF has been used to refer to a number of related but not necessarily equivalent concepts. To clarify our use of the term and the relationship of our decompositions of the resolvent and transition matrices of an MDP, we provide a brief discussion here; a summary is provided in Table~\ref{table:features}.

We will use $A$ to refer to the adjacency matrix of the unweighted, undirected graph induced by the matrix $P$ (i.e. $A[i,j]$ is 1 if there exists some action with nonzero probability of taking the agent from state $i$ to state $j$ or from state $j$ to state $i$, and 0 otherwise). $L_G$ will refer to the graph Laplacian based on this matrix $A$. 

We can additionally consider the Laplacian of the weighted, directed graph defined by $P^\pi$; we will refer to this matrix as $L_{P^\pi}$, in reference to its dependence on the probability of transitioning. $T$ denotes the matrix defined by a collection of sampled transitions indexed by $t$, with entries $T_{it} = -1$ if the transition $t$ leaves $i$ and $+1$ if it enters state $i$.

Our first observation is that eigendecomposition and SVD are equivalent for symmetric matrices because any real symmetric matrix has an orthogonal eigenbasis; this means that performing either decomposition yields the same eigenvectors and easily related eigenvalues. Our second observation is that when $P^\pi$ is \textit{not} symmetric, its singular value decomposition and eigendecomposition may diverge; further, the relationship between the SVD of the resolvent matrix $\Psi = (I - \gamma P^\pi)^{-1}$ and of $P^\pi$ is no longer straightforward, despite the eigenspaces of the two matrices being analogous. This means that analysis of the singular value decomposition of $P^\pi$ does not immediately imply any results about the resolvent matrix. 

\begin{table}[!h]
    \centering
    \begin{tabular}{c|c|c}
        Matrix & SVD & Eigendecomposition (ED)  \\
        \hline
        $L_G$ & PVFs \citep{mahadevan2007proto} & Equivalent to SVD  \\
        $T$ & sometimes $\equiv$ ED($L_G$) \citep{machado2017laplacian}  & not discussed \\
        $L_{P^\pi}$ & $\neq$ ED($L_{P^\pi}$)  & \citet{stachenfeld2014design} \\
        $(I - \gamma P^\pi)^{-1}$ &  RSBFs & $ \equiv L_{P^\pi}$ \\
        $P^\pi$ & \citet{behzadian2018feature} & $\equiv L_{P^\pi}$ \\
    \end{tabular}
    \caption{Summary of decompositions of various matrices associated with MDP transition operators, and associated features.}
    \vspace{0.5cm}
    \label{table:features}
\end{table}

Finally, we note that applying a uniform random walk policy may not be sufficient to guarantee that $P^\pi$ will be symmetric, and that in general it will not be possible to obtain a policy which will symmetrize the transition matrix. For example: when $G$ is a connected, non-regular graph (as is the case in many environments such as chains), there must be a node $v$ of degree $d$ adjacent to a node $v'$ of degree $d' \neq d$. A random walk policy will assign $p(v, v') = \frac{1}{d}$, while $p(v', v)$ will receive probability $\frac{1}{d'}$; thus, $P^\pi$ will not be symmetric. Fortunately, this is not a barrier to spectral analysis; the eigenvectors and eigenvalues of $P^\pi$ will still be real, as their transition matrix will be \textit{similar} to a symmetric matrix. We defer to \citet{machado2017laplacian} for a more detailed discussion of this relationship.

\section{Bayes-optimality of RSBFs} \label{sec:bayes-opt}

We can develop the discussion of RSBFs beyond their properties as a matrix decomposition described in Section~\ref{sec:feature-selection} to observe that the RSBFs characterize the Bayes-optimal features for predicting an unknown value function given an isotropic Gaussian prior distribution on the reward, and further characterize a Bayesian posterior over value functions given by conditioning on the known dynamics of the MDP. We will denote by $V_K(\Psi)$ the top $K$ eigenvectors of the matrix $\Psi \Psi^\top$, i.e. the top $K$ left singular vectors of $\Psi$.
\begin{restatable}{corollary}{corrBayesOpt}
Under an isotropic Gaussian prior on reward function $r \in \mathbb{R}^{\statespace}$, the subspace $V_K(\Psi)$ corresponds to the optimal subspace with respect to the following regression problem.
\begin{equation}
    \min_{\Phi \in \mathbb{R}^{\statespace \times K}} \mathbb{E}_{r \sim \mathcal{N}(0, I)} \left\lbrack \| \Pi_{\Phi^\perp} (I - \gamma P^\pi)^{-1} r \|^2 \right\rbrack \, ,
\end{equation}
where $\Pi_{\Phi^\perp}$ denotes orthogonal projection onto the orthogonal complement of $\Phi$.
\end{restatable}

\begin{proof}
Let $S$ denote some subspace $S \subset V$.
    \begin{align}
        \mathbb{E}[\|\Pi_s \Psi r\|^2] &= \mathbb{E}[r^\top \Psi^\top \Pi_s^\top \Pi_s \Psi r]
        \intertext{We note that for any real symmetric matrix $A$ we can rewrite $A = \sum \alpha_i v_i v_i^\top$.}
        \mathbb{E}[r^\top \Psi^\top \Pi_s^\top \Pi_S \Psi r] &= \mathbb{E}[ r^\top (\sum \alpha_i v_i v_i^\top) r] = \mathbb{E}[\sum \alpha_i (r^\top v_i) (v_i^\top r)] \\
        &= \mathbb{E}[\sum \alpha_i v_i^\top r r^\top v_i] = \sum \alpha_i v_i^\top \mathbb{E}[r r^\top] v_i \\
        &= \sum \alpha_i v_i^\top v_i = \text{Tr}(\Psi^\top \Pi_S^\top \Pi_S \Psi) =\text{Tr}(\Psi^\top \Pi_S \Psi)
        \intertext{Finally, we can re-express the minimization problem as follows}
        \text{argmin}_{S: \text{Dim}(S) = k} \text{Tr}(\Psi^\top(\Pi_{S^\perp})\Psi) &= \text{argmax}_{S:\text{Dim}(S) = k} \text{Tr}(\Psi^\top \Pi_S \Psi)
        \intertext{Now, because the subspace spanned by the top $k$ left-singular vectors $\{u_1, \dots, u_k\}$ of $\Psi$ is known to be the maximizer of the above equation, we finally obtain}
        &= \langle u_1, \dots, u_k \rangle =  V_K(\Psi) \; .
    \end{align}
\end{proof}

\begin{corollary}
The limiting distribution of $\Phi^M_t$ under the random cumulant auxiliary task described in Theorem~\ref{thm:distribution} is equivalent to the Bayesian posterior over value functions obtained by conditioning on the dynamics $P^\pi$, and given a prior distribution on the reward function equal to $\mathcal{N}(0, \Sigma)$.
\end{corollary}
\begin{proof}
Each column of $Z_\Sigma$ is sampled from an isotropic Gaussian distribution, and therefore each feature $\phi_i \overset{D}{=} (I - \gamma P^\pi) \epsilon_i$. It therefore suffices to show that under a suitable prior distribution, the distribution of $\phi_i$ is equal to a Bayesian posterior. 
For this, it suffices to show that such a posterior can be obtained by conditioning on the transition dynamics $P^\pi$, and looking at the induced pushforward measure on the reward distribution. Noting that $(I - \gamma P^\pi)$ is invertible, we then obtain the following prior over $V^\pi$, assuming an isotropic Gaussian prior on $p_r(r)$ and any arbitrary distribution over potential transition dynamics $p_\pi(P^\pi)$ which covers $\mathbb{R}^{|S| \times d}$.
\begin{align}
    P(V^\pi) &= \int_{(r, P^\pi)} \mathbbm{1}[(I - \gamma P^\pi)^{-1}r = V^\pi] dp_r(r)dp_{\pi}( P^\pi) 
    \intertext{We observe that the random variable $V^\pi$ has conditional distribution $P(V^\pi|P^\pi) = P((I - \gamma P^\pi)^{-1}r)$, whose density is proportional to $p_r( (I - \gamma P^\pi)V)$ by the change of variables formula.}
    P(V^\pi | P^\pi) &= c p_r(r =  (I - \gamma P^\pi)V^\pi) \\
    \intertext{ Because our prior over $r$ is equal to the initialization distribution of $\epsilon_i$, we obtain}
    &= c p_{\text{init}}(\epsilon_i = (I - \gamma P^\pi) V^\pi) \\
    \intertext{ which is precisely the limiting distribution $p_\infty$ of $\phi_i$ (again applying the change of variables formula).}
    &= p_\infty (\phi_i = (I - \gamma P^\pi)^{-1} \epsilon_i = V^\pi)
\end{align}
So we see that the limiting distribution of $\phi_i$ is equal to the prior over value functions conditioned on the transition dynamics.
\end{proof}
\section{Learning Dynamics for Ensemble Prediction}
\label{sec:ensemble-dynamics}

We provide some visualizations of the induced behaviour on features as a result of training an ensemble with multiple heads and zero reward, replicating the analysis of Section \ref{sec:reps}, to highlight how the eigendecomposition of $P^\pi$ affects the learned representations. We run our evaluations on the Four-Rooms Gridworld by initializing $\Phi \in \mathbb{R}^{105 \times 10}$ (i.e. $|\statespace| = 105$ and the number of features $d=10$) and simulating the ODE defined in Equation~\ref{eq:ensemble-phi-flow} for time $t=100$ with transition matrix $P^\pi$ defined by the uniform random policy on this Gridworld. In some cases, the features converged to zero quickly and so we show a final $t < 100$ to highlight the behaviour of the representation before it reaches zero.

We consider three variables which we permit to vary: the initialization scheme of features, in one case sampled from an isotropic Gaussian \texttt{rand} or from a randomly initialized 2-layer MLP \texttt{nn}); whether the weight matrix is fixed at initialization \texttt{fix} or permitted to follow the flow defined by Equation~\ref{eq:w-ode} \texttt{train}; and finally the number of `heads', \texttt{M}=1, 20, and 200. 

In Figure \ref{fig:ensemble_predictionsl}, we plot the output of an arbitrary head $\mathbf{w}^m$ of the ensemble. In Figure \ref{fig:ensemble_feature0} we visualize the value of a single feature (i.e. a single column of $\Phi$). In Figure~\ref{fig:dotproduct_ensemble}, we track the dot product of the columns of $\Phi$ with the eigenfunctions of $P^\pi$.

We observe, as predicted, that for fixed heads in the overparameterized regime, the features (and the value functions they induce) converge to smooth eigenfunctions. We do not see meaningful convergence of the features trained in conjunction with a single weight vector. In contrast, the value functions and features trained in conjunction with ensembles with more heads than the feature dimension consistently resemble the eigenfunctions of $P^\pi$. When $(\mathbf{w}^m)$ are held fixed, we see convergence to smooth eigenfunctions as predicted by our theory; when $(\mathbf{w}^m)$ are permitted to vary according to the flow in Equation~\ref{eq:ensemble-w-flow}, we see convergence to the most eigenfunction corresponding to the most negative eigenfunction of $P^\pi$. We can observe the evolution of the dot product between the features and the EBFs of $P^\pi$ more clearly in Figure~\ref{fig:dotproduct_ensemble}. Here, each red line corresponds to the dot product between a feature and an EBF. The colour of the line indicates the order $i$ of the eigenvalue $\lambda_i$ to which the EBF corresponds, interpolating between red $\lambda_1$ and blue $\lambda_{105}$. Lower values of $i$ correspond to smoother eigenfunctions. We observe that for sufficiently large $M$, the representations exhibit higher dot product with the smoother eigenfunctions, while for $M=1$ the features stay largely fixed during training.

\begin{figure}[!h]
    \centering
    \includegraphics[width=0.75\linewidth]{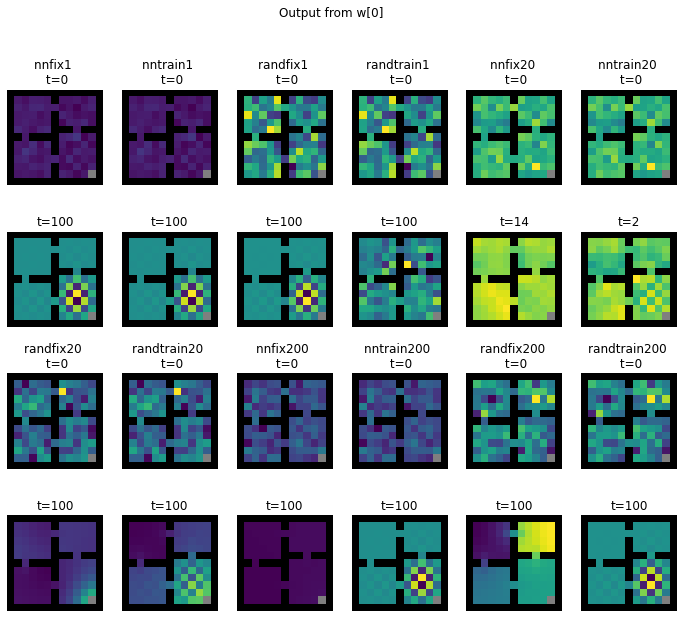}
    \caption{Value functions learned by the ensemble head at index 0 for different training regimes. Plot titles of form (feature initialization scheme, train/fix weight matrix, number of heads in ensemble). Observe that the representation learned with fixed weights tends to converge to smoother eigenfunctions than those learned with weights that are also allowed to train.}
    \label{fig:ensemble_predictionsl}
\end{figure}
\begin{figure}[!h]
    \centering
    \includegraphics[width=0.75\linewidth]{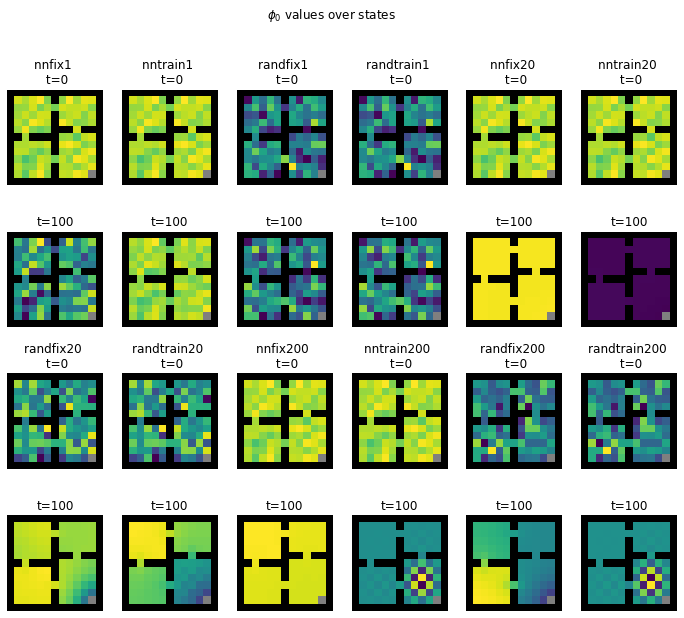}
    \caption{Values of ensemble feature at index 0 for different training regimes. Plot titles of form (feature initialization scheme, train/fix weight matrix, number of heads in ensemble). Observe that the representation learned with fixed weights tends to converge to smoother eigenfunctions than those learned with weights that are also allowed to train.}
    \label{fig:ensemble_feature0}
\end{figure}
\begin{figure}[!h]
    \centering
    \includegraphics[width=0.75\linewidth]{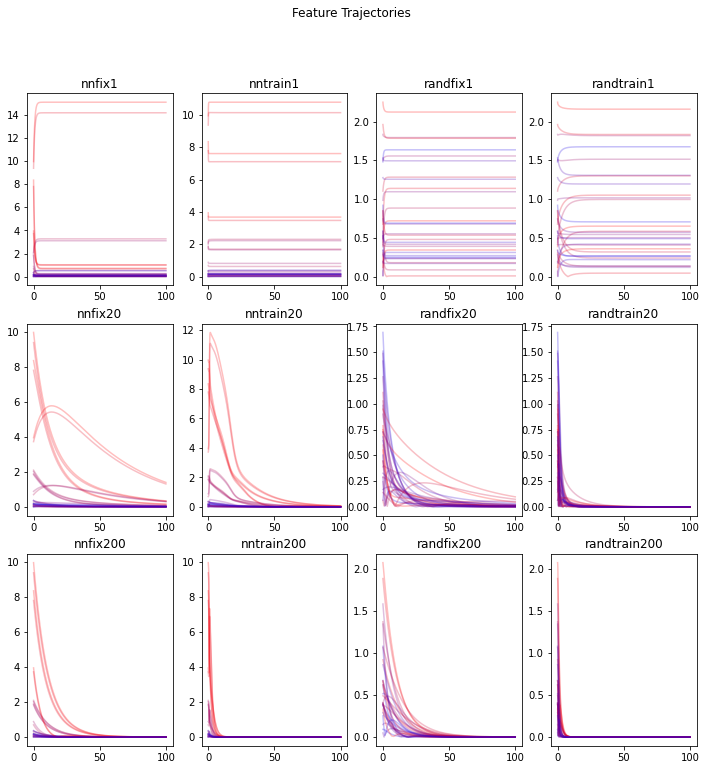}
    \caption{Projection of features onto eigenvectors of $P^\pi$. Red lines correspond to projection onto eigenvectors of higher eigenvalues, blue lines to lower eigenvalues.}
    \label{fig:dotproduct_ensemble}
\end{figure}

\end{appendix}

\end{document}